\documentclass[letterpaper]{article} 
\usepackage{aaai24}  
\usepackage{times}  
\usepackage{helvet}  
\usepackage{courier}  
\usepackage[hyphens]{url}  
\usepackage{graphicx} 
\usepackage{natbib}  
\usepackage{caption} 
\usepackage{algorithm}
\usepackage{algorithmic}
\usepackage{graphicx}
\usepackage{amsmath}
\usepackage{amssymb}
\usepackage{mathtools}
\usepackage{amsthm}
\usepackage{thm-restate}

\newcommand{\argmax}{\mathop{\rm arg~max}\limits}
\newcommand{\argmin}{\mathop{\rm arg~min}\limits}

\newtheorem{theorem}{Theorem}
\newtheorem{proposition}[theorem]{Proposition}
\newtheorem{lemma}[theorem]{Lemma}

\theoremstyle{definition}

\theoremstyle{remark}

\usepackage{newfloat}
\usepackage{listings}
\DeclareCaptionStyle{ruled}{labelfont=normalfont,labelsep=colon,strut=off} 
\floatstyle{ruled}
\newfloat{listing}{tb}{lst}{}
\floatname{listing}{Listing}
\title{Thompson Sampling for Real-Valued Combinatorial Pure Exploration of Multi-Armed Bandit}
\author {
    Shintaro Nakamura\textsuperscript{\rm 1,\rm 2},
    Masashi Sugiyama\textsuperscript{ \rm 2, \rm 1},
}
\affiliations {
    \textsuperscript{\rm 1}The University of Tokyo\\
    \textsuperscript{\rm 2}RIKEN AIP\\
    nakamurashintaro@g.ecc.u-tokyo.ac.jp, sugi@k.u-tokyo.ac.jp
}

\usepackage{bibentry}

\begin{document}

\maketitle

\begin{abstract}
    We study the real-valued combinatorial pure exploration of the multi-armed bandit (R-CPE-MAB) problem. In R-CPE-MAB, a player is given $d$ stochastic arms, and the reward of each arm $s\in\{1, \ldots, d\}$ follows an unknown distribution with mean $\mu_s$. In each time step, a player pulls a single arm and observes its reward. The player's goal is to identify the optimal \emph{action} $\boldsymbol{\pi}^{*} = \argmax_{\boldsymbol{\pi} \in \mathcal{A}} \boldsymbol{\mu}^{\top}\boldsymbol{\pi}$ from a finite-sized real-valued \emph{action set} $\mathcal{A}\subset \mathbb{R}^{d}$ with as few arm pulls as possible. Previous methods in the R-CPE-MAB assume that the size of the action set $\mathcal{A}$ is polynomial in $d$. We introduce an algorithm named the Generalized Thompson Sampling Explore (GenTS-Explore) algorithm, which is the first algorithm that can work even when the size of the action set is exponentially large in $d$. We also introduce a novel problem-dependent sample complexity lower bound of the R-CPE-MAB problem, and show that the GenTS-Explore algorithm achieves the optimal sample complexity up to a problem-dependent constant factor. 
\end{abstract}

\section{Introduction}\label{IntroductionSection}
Pure exploration in the stochastic multi-armed bandit (PE-MAB) is one of the important frameworks for investigating online decision-making problems, where we try to identify the optimal object from a set of candidates as soon as possible \citep{Bubeck2009,Audibert2010,SChen2014}.
One of the important models in PE-MAB is the \emph{combinatorial pure exploration} task in the multi-armed bandit (CPE-MAB) problem \citep{SChen2014,WangAndZhu2022,Gabillon16,LChen2016,LChen2017}. In CPE-MAB, we have a set of $d$ stochastic arms, where the reward of each arm $s\in\{1, \ldots, d\}$ follows an unknown distribution with mean $\mu_s$, and a finite-sized \emph{action set} $\mathcal{A}$, which is a collection of subsets of arms with certain combinatorial structures. The size of the action set can be exponentially large in $d$. In each time step, a player pulls a single arm and observes a reward from it. The goal is to identify the best action from action set $\mathcal{A}$ with as few arm pulls as possible.
Abstractly, the goal is to identify $\boldsymbol{\pi}^{*}$, which is the optimal solution for the following constraint optimization problem:
\begin{equation}
\begin{array}{ll@{}ll}
\text{maximize}_{\boldsymbol{\pi}}  & \boldsymbol{\mu}^{\top}\boldsymbol{\pi}&\\
\text{subject to}& \boldsymbol{\pi}\in \mathcal{A},   & 
\end{array}\label{AbstractFormulation}
\end{equation}
where $\boldsymbol{\mu}$ is a vector whose $s$-th element is the mean reward of arm $s$ and $\top$ denotes the transpose. One example of the CPE-MAB is the shortest path problem shown in Figure~\ref{ShortestPathFigure}. Each edge $s \in \{1, \ldots, 7\}$ has a cost $\mu_s$ and $\mathcal{A} = \{(1, 0, 1, 0, 0, 1, 0), (0, 1, 0, 1, 0, 1, 0), (0, 1, 0, 0, 1, 0, 1)\}$. In real-world applications, the cost of each edge (road) can often be a random variable due to some traffic congestion, and therefore the cost stochastically changes. We assume we can choose an edge (road) each round, and conduct a traffic survey for that edge (road). If we conduct a traffic survey, we can observe a random sample of the cost of the chosen edge. Our goal is to identify the best action, which is a path from the start to the goal nodes.\par
Although CPE-MAB can be applied to many models which can be formulated as (\ref{AbstractFormulation}), most of the existing works in CPE-MAB \citep{SChen2014,WangAndZhu2022,Gabillon16,LChen2017,YihanDu2021,LChen2016} assume $\mathcal{A}\subseteq\{0, 1\}^d$. This means that the player's objective is to identify the best action which maximizes the sum of the expected rewards. Therefore, although we can apply the existing CPE-MAB methods to the shortest path problem \citep{Sniedovich2006}, top-$K$ arms identification \citep{Kalyanakrishnan2010}, matching \citep{Gibbons1985}, and spanning trees \citep{Pettie2002}, we cannot apply them to problems where $\mathcal{A} \subset \mathbb{R}^d$, such as the optimal transport problem \citep{villani2008}, the knapsack problem \citep{dantzig2007}, and the production planning problem \citep{Pochet2010}. For instance, the optimal transport problem shown in Figure \ref{OT_figure} has a real-valued action set $\mathcal{A}$. We have five suppliers and four demanders. Each supplier $i$ has $s_i$ goods to supply. Each demander $j$ wants $d_j$ goods. Each edge $\mu_{ij}$ is the cost to transport goods from supplier $i$ to demander $j$. Our goal is to minimize $\sum_{i = 1}^5 \sum_{j = 1}^4 \pi_{ij}\mu_{ij}$, where $\pi_{ij}(\geq 0)$ is the amount of goods transported to demander $j$ from supplier $i$. Again, we assume that we can choose an edge (road) each round, and conduct a traffic survey for that edge. Our goal is to identify the best action, which is a transportation plan (matrix) that shows how much goods each supplier should send to each demander. \par
To overcome the limitation of the existing CPE-MAB methods, \citet{Nakamura2023} has introduced a real-valued CPE-MAB (R-CPE-MAB), where the action set $\mathcal{A} \subset \mathbb{R}^d$. However, it needs an assumption that the size of the action set $\mathcal{A}$ is polynomial in $d$, which is not satisfied in general since in many combinatorial problems, the action set is exponentially large in $d$. To cope with this problem, one may leverage algorithms from the \emph{transductive linear bandit} literature \citep{Fiez2019,KatzSamuels2020} for the R-CPE-MAB.
In the transductive bandit problem, a player chooses a \emph{probing vector} $\boldsymbol{v}$ from a given finite set $\mathcal{X} \subset \mathbb{R}^d$ each round and observes $\boldsymbol{\mu}^{\top} \boldsymbol{v} + \epsilon$, where $\epsilon$ is a noise from a certain distribution. Her goal is to identify the best \emph{item} $\boldsymbol{z}^{*}$ from a finite-sized set $\mathcal{Z} \subset\mathbb{R}^d$, which is defined as $\boldsymbol{z}^{*} = \argmax_{\boldsymbol{z}\in \mathcal{Z}}\boldsymbol{\mu}^{\top} \boldsymbol{z}$. 
The transductive linear bandit can be seen as a generalization of the R-CPE-MAB since the probing vectors are the standard basis vectors and the items are the actions in the R-CPE-MAB. 
However, the RAGE algorithm introduced in \citet{Fiez2019} has to enumerate all the items in $\mathcal{Z}$, and therefore, not realistic to apply it when the size of $\mathcal{Z}$ is exponentially large in $d$. The Peace algorithm \citep{KatzSamuels2020} is introduced as an algorithm that can be applied to the CPE-MAB even when the size of $\mathcal{Z}$ is exponentially large in $d$, but it cannot be applied to the R-CPE-MAB since its subroutine that determines the termination of the algorithm is only valid when $\mathcal{Z} \subset \left\{0, 1\right\}^{d}$. \par
In this study, we introduce an algorithm named the Generalized Thompson Sampling Explore (GenTS-Explore) algorithm, which can identify the best action in the R-CPE-MAB even when the size of the action set is exponentially large in $d$. This algorithm can be seen as a generalized version of the Thompson Sampling Explore (TS-Explore) algorithm introduced by \citet{WangAndZhu2022}. \par
Additionally, we show lower bounds of the R-CPE-MAB. One is written explicitly; the other is written implicitly and tighter than the first one. We introduce a hardness measure $\mathbf{H} = \sum_{s = 1}^{d} \frac{1}{\Delta^2_{(s)}}$, where $\Delta_{(s)}$ is named \emph{G-Gap}, which can be seen as a generalization of the notion \emph{gap} introduced in the CPE-MAB literature \citep{SChen2014,LChen2016,LChen2017}. We show that the sample complexity upper bound of the Gen-TS-Explore algorithm matches the lower bound up to a factor of a problem-dependent constant term. \par
\begin{figure}[t]
  \begin{minipage}[b]{0.4\linewidth}
    \centering
    \includegraphics[width = \linewidth]{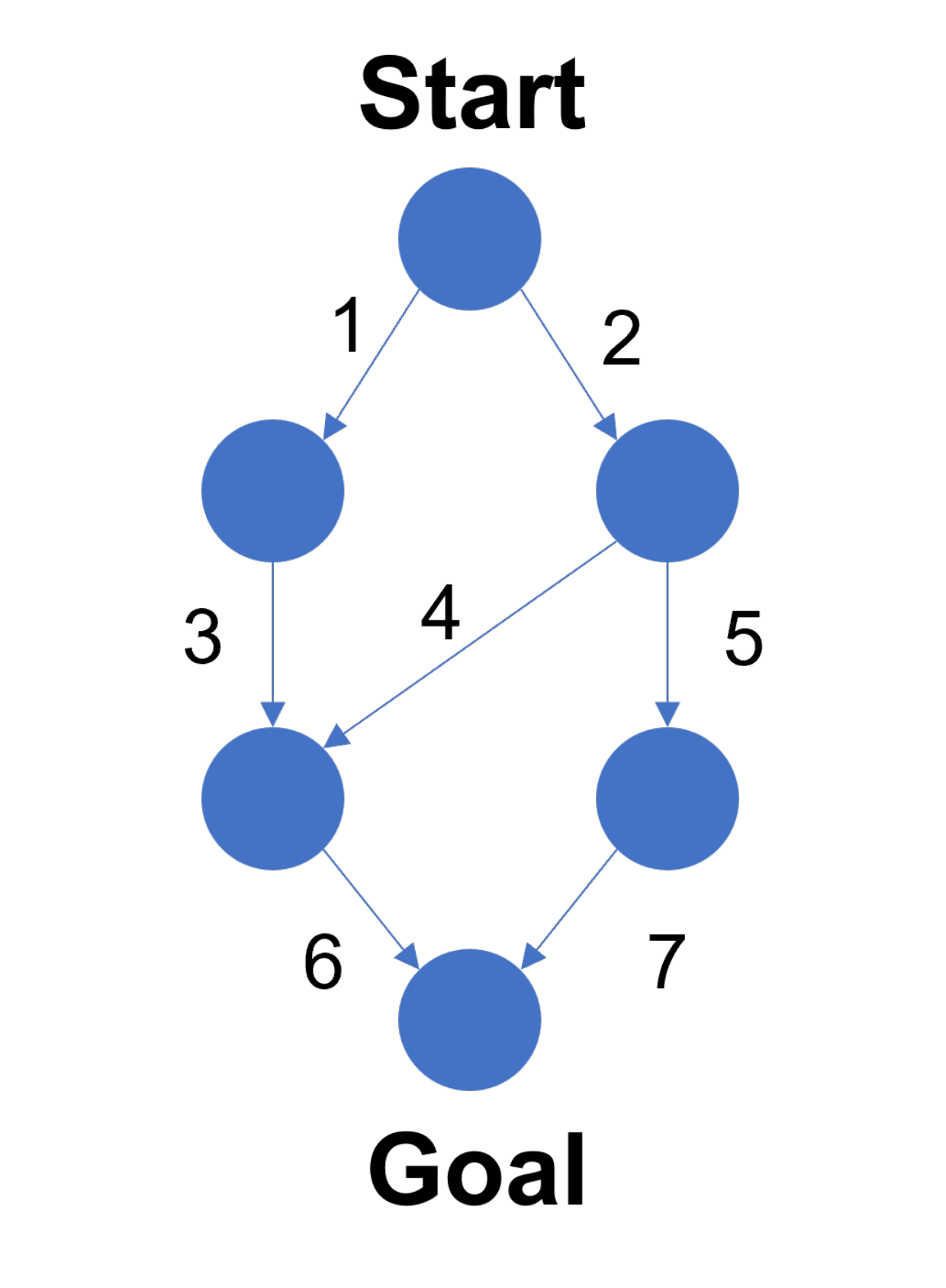}
    \caption{A schematic of the shortest path problem.}
    \label{ShortestPathFigure}
  \end{minipage}
  \begin{minipage}[b]{0.55\linewidth}
    \centering
    \includegraphics[width = \linewidth]{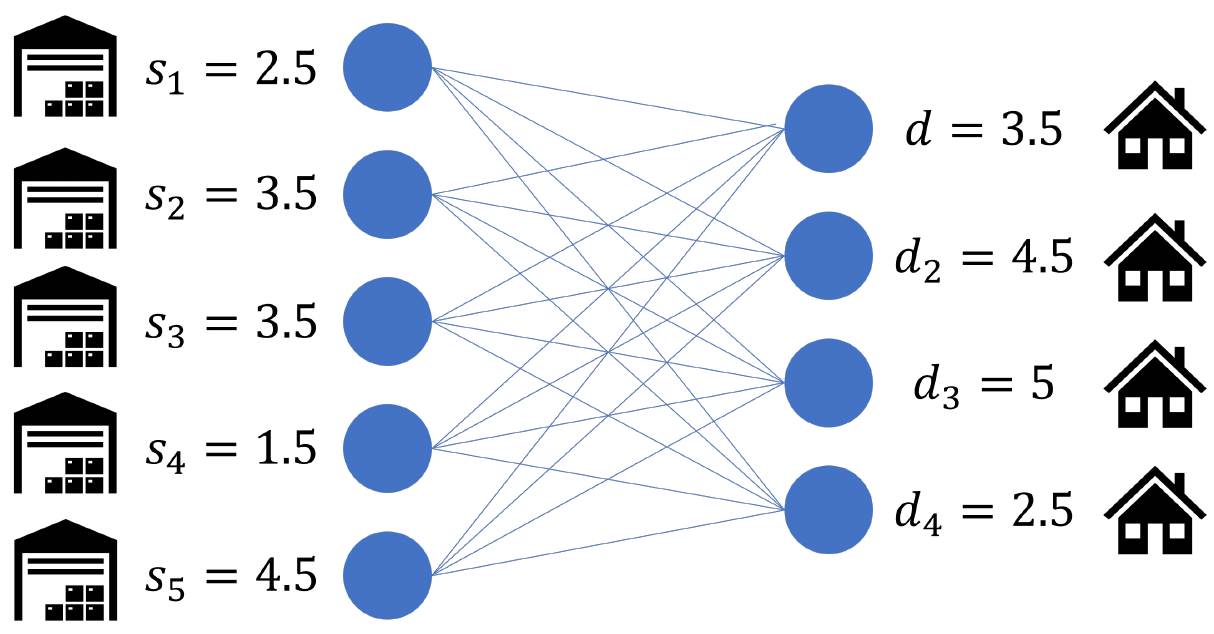}
    \caption{A schematic of the optimal transport problem. One candidate of $\boldsymbol{\pi}$ can be 
$\boldsymbol{\pi} = \begin{pmatrix}
2.5 & 0 & 0 & 0 \\
1.0 & 2.5 & 0 & 0 \\
0 & 2.0 & 1.5 & 0 \\
0 & 0 & 1.5 & 0 \\
0 & 0 & 2.0 & 2.5 \\
\end{pmatrix}.$ }
\label{OT_figure}
  \end{minipage}
\end{figure}
\section{Problem Formulation}\label{ProblemFormulation}
In this section, we formally define the R-CPE-MAB model similar to \citet{SChen2014}. Suppose we have $d$ arms, numbered $1, \ldots, d$. Assume that each arm $s\in[d]$ is associated with a reward distribution $\phi_s$, where $[d] = \{1, \ldots, d\}$. We assume all reward distributions have an $R$-sub-Gaussian tail for some known constant $R>0$. Formally, if $X$ is a random variable drawn from $\phi_s$, then, for all $\lambda\in \mathbb{R}$, we have $\mathbb{E}[\exp (\lambda X - \lambda \mathbb{E}[X])] \leq \exp (R^2\lambda^2/2)$. It is known that the family of $R$-sub-Gaussian tail distributions includes all distributions that are supported on $[0, R]$ and also many unbounded distributions such as Gaussian distributions with variance $R^2$ \citep{Rivasplata2012}. We denote by $\mathcal{N}(\mu, \sigma^2)$ the Gaussian distribution with mean $\mu$ and variance $\sigma^2$.
Let $\boldsymbol{\mu} = (\mu_1, \ldots, \mu_d)^{\top}$ denote the vector of expected rewards, where each element $\mu_s = \mathbb{E}_{X\sim\phi_s}[X]$ denotes the expected reward of arm $s$ and $\top$ denotes the transpose. We denote by $T_s(t)$ the number of times arm $s$ is pulled before round $t$, and by $\boldsymbol{\hat{\mu}}(t) = (\hat{\mu}_1(t), \ldots, \hat{\mu}_d(t))^{\top}$ the vector of sample means of each arm before round $t$.\par
With a given $\boldsymbol{\nu}$, let us consider the following linear optimization problem:
\begin{equation}\label{Abstract_Formulation_IntegerProgramminProblem}
\begin{array}{ll@{}ll}
\text{maximize}_{\boldsymbol{\pi}}  & \boldsymbol{\nu}^{\top}\boldsymbol{\pi}&\\
\text{subject to}& \boldsymbol{\pi}\in \mathcal{C}\subset\mathbb{R}^d,   &
\end{array}
\end{equation}
where $\mathcal{C}$ is a problem-dependent feasible region. For any $\boldsymbol{\nu}\in \mathbb{R}^d$, we denote $\boldsymbol{\pi}^{\boldsymbol{\nu}, \mathcal{C}}$ as the optimal solution of (\ref{Abstract_Formulation_IntegerProgramminProblem}). 
Then, we define the action set $\mathcal{A}$ as the set of vectors that contains optimal solutions of (\ref{Abstract_Formulation_IntegerProgramminProblem}) for any $\boldsymbol{\nu}$, i.e., 
\begin{equation}
    \mathcal{A} = \left\{ \boldsymbol{\pi}^{\boldsymbol{\nu, \mathcal{C}}} \in \mathbb{R}^d \ | \ \forall \boldsymbol{\nu}\in \mathbb{R}^d\right\}.
\end{equation}
 Note that $K$ could be exponentially large in $d$.
The player's objective is to identify $\boldsymbol{\pi}^{*} = \argmax_{\boldsymbol{\pi}\in\mathcal{A}} \boldsymbol{\mu}^{\top}\boldsymbol{\pi}$ by playing the following game. At the beginning of the game, the action set $\mathcal{A}$ is revealed. Then, the player pulls an arm over a sequence of rounds; in each round $t$, she pulls an arm $p_t\in[d]$ and observes a reward sampled from the associated reward distribution $\phi_{p_t}$. The player can stop the game at any round. She needs to guarantee that $\Pr\left[  \boldsymbol{\pi}_{\mathrm{out}} \neq \boldsymbol{\pi}^{*}  \right] \leq \delta$ for a given confidence parameter $\delta$. For any $\delta \in (0, 1)$, we call an algorithm $\mathbb{A}$ a $\delta$-correct algorithm if, for any expected reward $\boldsymbol{\mu}\in\mathbb{R}$, the probability of the error of $\mathbb{A}$ is at most $\delta$, i.e., $\Pr\left[ \boldsymbol{\pi}_{\mathrm{out}} \neq \boldsymbol{\pi}^{*} \right] \leq \delta$.
The learner's performance is evaluated by her \emph{sample complexity}, which is the round she terminated the game. We assume $\boldsymbol{\pi}^{*}$ is unique.
\subsection{Technical Assumptions}
To cope with the exponential largeness of the action set, we make two mild assumptions for our R-CPE-MAB model. The first one is the existence of the \emph{offline oracle}, which computes $\boldsymbol{\pi}^{*}(\boldsymbol{\nu}) = \argmax_{\boldsymbol{\pi} \in \mathcal{A}} \boldsymbol{\nu}^{\top} \boldsymbol{\pi}$ in polynomial or pseudo-polynomial time once $\boldsymbol{\nu}$ is given. We write $\mathrm{Oracle}(\boldsymbol{\nu}) = \boldsymbol{\pi}^{*}(\boldsymbol{\nu}) $. This assumption is relatively mild since in linear programming, we have the network simplex algorithm \citep{Nelder65} and interior points methods \citep{Karmakar1984}, whose computational complexities are both polynomials in $d$. Moreover, if we consider the knapsack problem, though the knapsack problem is NP-complete \citep{Garey1979} and is unlikely that it can be solved in polynomial time, it is well known that we can solve it in pseudo-polynomial time if we use dynamic programming \citep{Kellerer2011,Fujimoto2016}. In some cases, it may be sufficient to use this dynamic programming algorithm as the offline oracle in the R-CPE-MAB. \par
The second assumption is that the set of possible outputs of the offline oracle is finite-sized. This assumption also holds in many combinatorial optimization problems. For instance, no matter what algorithm is used to compute  the solution to the knapsack problem, the action set is a finite set of integer vectors, so this assumption holds. 
Also, in linear programming problems such as the optimal transport problem \citep{villani2008} and the production planning problem \citep{Pochet2010}, it is well known that the solution is on a vertex of the feasible region, and therefore, the set of candidates of solutions for optimization problem (\ref{AbstractFormulation}) is finite. \par

\section{Lower Bound of R-CPE-MAB} \label{LowerBoundSection}
In this section, we discuss sample complexity lower bounds of R-CPE-MAB. In Theorem \ref{LowerBoundTheorem_Explicit}, we show a sample complexity lower bound which is derived explicitly. In Theorem \ref{LowerBoundTheoremImplicit}, we show another lower bound, which is only written in an implicit form but is tighter than that in Theorem \ref{LowerBoundTheorem_Explicit}. \par
In our analysis, we have several key quantities that are useful to discuss the sample complexity upper bounds.
First, we define $\boldsymbol{\pi}^{(s)}$ as follows:
\begin{equation}
    \boldsymbol{\pi}^{(s)} = \argmin_{\boldsymbol{\pi} \in \mathcal{A} \setminus \{ \boldsymbol{\pi}^{*} \}} \frac{\boldsymbol{\mu}^{\top} \left( \boldsymbol{\pi}^{*} - \boldsymbol{\pi} \right) }{ \left| \pi^{*}_{s} - \pi_{s} \right| }.
\end{equation}
Intuitively, among the actions whose $s$-th element is different from $\boldsymbol{\pi}^{*}$, $\boldsymbol{\pi}^{(s)}$ is the one that is the most difficult to confirm its optimality.
We define a notion named \emph{G-gap} which is formally defined as follows:
\begin{eqnarray}
    \Delta_{(s)}& = &  \frac{\boldsymbol{\mu}^{\top} (\boldsymbol{\pi}^{*} - \boldsymbol{\pi}^{(s)})}{|\pi^{*}_s - \pi^{(s)}_s|} \nonumber\\
    & = & \min_{\boldsymbol{\pi} \in \mathcal{A} \setminus \left\{ \boldsymbol{\pi}^{*} \right\}} \frac{\boldsymbol{\mu}^{\top} \left(\boldsymbol{\pi}^{*} - \boldsymbol{\pi} \right) }{|\pi^{*}_s - \pi_s|}.
\end{eqnarray}
\emph{G-gap} can be seen as a natural generalization of \emph{gap} introduced in the CPE-MAB literature \citep{SChen2014,LChen2016,LChen2017}.
Then, we denote the sum of inverse squared gaps by 
\begin{eqnarray}
    \mathbf{H} &=& \sum_{s = 1}^{d} \left(\frac{1}{\Delta_{(s)}} \right)^{2} \nonumber \\
               &=& \sum_{s = 1}^{d} \max_{\boldsymbol{\pi} \in \mathcal{A} \setminus \{ \boldsymbol{\pi}^{*} \} } \frac{\left| \pi^{*}_s - \pi_s  \right|^2}{ \left(\left( \boldsymbol{\pi}^{*} - \boldsymbol{\pi} \right)^{\top} \boldsymbol{\mu}\right)^2 }, \nonumber
\end{eqnarray}
which we define as a hardness measure of the problem instance in R-CPE-MAB. In Theorem \ref{LowerBoundTheorem_Explicit}, we show that $\mathbf{H}$ appears in a sample complexity lower bound of R-CPE-MAB. Therefore, we expect that this quantity plays an essential role in characterizing the difficulty of the problem instance. 
\subsection{Explicit Form of a Sample Complexity Lower Bound} 
Here, we show a sample complexity lower bound of the R-CPE-MAB which is written in an explicit form.

\begin{restatable}[]{theorem}{LowerBoundTheoremExplicit} \label{LowerBoundTheorem_Explicit}
    Fix any action set $\mathcal{A} \subset \mathbb{R}^{d}$ and any vector $\boldsymbol{\mu} \in \mathbb{R}^d$. Suppose that, for each arm $s\in [d]$, the reward distribution $\phi_s$ is given by $\phi_s = \mathcal{N}(\mu_s, 1)$. Then, for any $\delta \in \left( 0, \frac{e^{-16}}{4} \right)$ and any $\delta$-correct algorithm $\mathbb{A}$, we have 
    \begin{equation}\label{LowerBoundTheorem_Explicit_Equation}
        \mathbb{E}\left[ T \right] \geq \frac{1}{16} \mathbf{H}\log \left( \frac{1}{4\delta} \right),
    \end{equation}
    where $T$ denotes the total number of arm pulls by algorithm $\mathbb{A}$.
\end{restatable}
Theorem \ref{LowerBoundTheorem_Explicit} can be seen as a natural generalization of the result in ordinary CPE-MAB shown in \citet{SChen2014}.
In the CPE-MAB literature, the hardness measure $\mathbf{H}'$ is defined as follows \citep{SChen2014,WangAndZhu2022,LChen2017}:
\begin{equation}
    \mathbf{H}' = \sum_{s = 1}^{d} \left(\frac{1}{\Delta_{s}} \right)^{2},
\end{equation}
where 
\begin{equation}
    \Delta_{s} = \min_{\boldsymbol{\pi} \in \{ \boldsymbol{\pi}\in\mathcal{A} \ | \ \pi_{s} \neq \pi^{*}_{s} \}} {\boldsymbol{\mu}^{\top} \left( \boldsymbol{\pi}^{*} - \boldsymbol{\pi} \right) }.
\end{equation}
Below, we discuss why the hardness measure in R-CPE-MAB uses $\Delta_{(s)}$ not $\Delta_{s}$. \par
Suppose we have two bandit instances $\mathcal{B}_1$ and $\mathcal{B}_2$. In $\mathcal{B}_1$, $\mathcal{A}_1 = \left\{ \left( 100, 0 \right)^{\top}, (0, 100)^{\top} \right\}$ and $\boldsymbol{\mu}_1 = \left( \mu_{1, 1}, \mu_{1, 2} \right) = \left( 0.011, 0.01 \right)^{\top}$. In $\mathcal{B}_2$, $\mathcal{A}_2 = \left\{ \left( 1, 0 \right)^{\top}, \left( 0, 1 \right)^{\top} \right\}$ and $\boldsymbol{\mu}_{2} = \left( \mu_{2, 1}, \mu_{2, 2} \right) = \left( 0.1, 0.11 \right)^{\top}$. We assume that, for both instances, the arms are equipped with Gaussian distributions with unit variance. Also, for any $i \in \{1, 2\}$ and $s\in\{ 1, 2 \}$, let us denote by $T_{i, s}(t)$ the number of times arm $s$ is pulled in the bandit instance $\mathcal{B}_{i}$ in round $t$. Let us consider the situation where $T_{1, 1}(t) = T_{2, 1}(t)$ and $T_{1, 2}(t) = T_{2, 2}(t)$, and we have prior knowledge that $\mu_{1, 1} \in \left[ \hat{\mu}_{1, 1} - \sigma_{1}, \hat{\mu}_{1,1} + \sigma_{1} \right]$, $\mu_{1, 2} \in \left[ \hat{\mu}_{1, 2} - \sigma_{2}, \hat{\mu}_{1,2} + \sigma_{2} \right]$, $\mu_{2, 1} \in \left[ \hat{\mu}_{2, 1} - \sigma_{1}, \hat{\mu}_{2,1} + \sigma_{1} \right]$, and $\mu_{2, 2} \in \left[ \hat{\mu}_{2, 2} - \sigma_{2}, \hat{\mu}_{2,2} + \sigma_{2} \right]$. Here, $\sigma_1$ and $\sigma_2$ are some confidence bounds on the rewards of arms, which may be derived by some concentration inequality. Note that they depend only on the number of times the arm is pulled, and that the confidence bound for each arm is the same in both instances since $T_{1, 1}(t) = T_{2, 1}(t)$ and $T_{1, 2}(t) = T_{2, 2}(t)$. \par
We can see that $\mathbf{H}'$ are the same in both $\mathcal{B}_1$ and $\mathcal{B}_2$, which implies that the difficulty in identifying the best actions is the same in $\mathcal{B}_1$ and $\mathcal{B}_2$. However, this is not true since when we estimate the reward of actions in $\mathcal{B}_1$, the confidence bound will be amplified by 100, and therefore, we are far less confident to determine the best action in $\mathcal{B}_1$ than $\mathcal{B}_2$. On the other hand, $\mathbf{H}$ reflects this fact. $\mathbf{H}$ in $\mathcal{B}_1$ is 10000 larger than that of $\mathcal{B}_2$, which implies that identifying the best action in $\mathcal{B}_1$ is much more difficult than $\mathcal{B}_2$. \par

\subsection{Implicit Form of a Lower Bound}
Here, in Theorem \ref{LowerBoundTheoremImplicit}, we show that we can generalize the tightest lower bound in the CPE-MAB literature shown in \citet{LChen2017} for the R-CPE-MAB. 
\begin{restatable}[]{theorem}{LowerBoundTheoremImplicit} \label{LowerBoundTheoremImplicit}
    For any $\delta \in (0, 0.1)$ and a $\delta$-correct algorithm $\mathbb{A}$, $\mathbb{A}$ will pull arms $\Omega(\mathrm{Low}(\mathcal{A}) \log \frac{1}{\delta})$ times, where $\mathrm{Low}(\mathcal{A})$ is the optimal value of the following mathematical program:
    \begin{equation} \label{Low(A)OptProblem}
        \begin{aligned}
        &\mathrm{minimize}  &&\sum_{s = 1}^{d}\tau_s  \\ 
        &\mathrm{subject \ to} &&
        \forall \boldsymbol{\pi} \in \mathcal{A},   \sum_{s \in \boldsymbol{\pi}^{*} \diamond\boldsymbol{\pi}} \frac{\left| \pi^{*}_s - \pi_s \right|^2}{\tau_s}  \leq \Delta^2_{\boldsymbol{\pi}^{*}, \boldsymbol{\pi}} &&&  \\
            & && \tau_s > 0, \forall s\in [d],
        \end{aligned}
    \end{equation}
    where $\boldsymbol{\pi}^{*} \diamond\boldsymbol{\pi} = \{ s \in [d] \ | \ \pi^{*}_s \neq \pi_s \}$ and ${\Delta_{\boldsymbol{\pi}^{*}, \boldsymbol{\pi}}} = \boldsymbol{\mu}^{\top}\left( \boldsymbol{\pi}^{*} - \boldsymbol{\pi} \right) $.
\end{restatable}
In the appendix, we show that the lower bound in Theorem \ref{LowerBoundTheoremImplicit} is no weaker than that in Theorem \ref{LowerBoundTheorem_Explicit} by showing $\mathrm{Low}(\mathcal{A}) \geq \mathbf{H}$. \par
This lower bound is exactly equal to the lower bound in \citet{Fiez2019} for the transductive bandit, which is written as follows:
\begin{eqnarray}
    \rho_{*} \log\left( \frac{1}{\delta} \right),
\end{eqnarray}
where 
\begin{eqnarray}
    \rho_{*} = \min_{\boldsymbol{\lambda} \in \Pi_{d}} \max_{\boldsymbol{\pi} \in \mathcal{A} \setminus \{\boldsymbol{\pi}^{*}\}} \frac{\sum_{s = 1}^{d} \frac{\left| \pi^{*}_{s} - \pi_{s} \right|^{2}}{\lambda_{s}}}{\Delta^{2}_{\boldsymbol{\pi}^{*}, \boldsymbol{\pi}}}. \label{rho_definition}
\end{eqnarray}

\section{GenTS-Explore Algorithm} \label{GenTS-ExploreAlgorithmSection}
In this section, we introduce an algorithm named the Generalized Thompson Sampling Explore (GenTS-Explore) algorithm, which can identify the best action in the R-CPE-MAB even when the size of the action set $\mathcal{A}$ is exponentially large in $d$. We first explain what the GenTS-Explore algorithm is doing at a high level. Then, we show a sample complexity upper bound of it.
\subsection{Outline of the GenTS-Explore Algorithm}
Here, we show what the GenTS-Explore algorithm is doing at a high level (Algorithm \ref{GenTS-ExploreAlgorithm}). \par 
The GenTS-Explore algorithm can be seen as a modified version of the TS-Explore algorithm introduced in \citet{WangAndZhu2022}. At each round $t$, it outputs $\hat{\boldsymbol{\pi}}(t) = \mathrm{Oracle}(\hat{\boldsymbol{\mu}}(t))$, which is the \emph{empirically best} action (line \ref{EmpiricalBestLine}). 
\begin{algorithm}[]
   \caption{GenTS-Explore Algorithm}
   \label{GenTS-ExploreAlgorithm}
\begin{algorithmic}[1]
   \STATE {\bfseries Input:} Confidence level $\delta$, $q\in[\delta, 0.1]$, $t \leftarrow 0$, $\boldsymbol{T}(0) = (T_1(0), \ldots, T_d(0))^{\top} = (0, \ldots, 0)^{\top}$ 
   \STATE {\bfseries Output:} Action $\boldsymbol{\pi}_{\mathrm{out}} \in \mathcal{A}$
   \STATE \texttt{// Initialization}
   \STATE Pull each arm once, and update their number of pulls $T_i$'s and the $\hat{\mu}_s(t)$
   \STATE $t\leftarrow d$
   \WHILE{\textbf{true}}
   \STATE $\hat{\boldsymbol{\pi}}(t) \leftarrow \mathrm{Oracle}(\hat{\boldsymbol{\mu}}(t))$ \label{EmpiricalBestLine}
   \FOR{$k = 1, \ldots, M(\delta, q, t)$}
   \STATE For each arm $s$, draw $\theta_s^k(t)$ independently from distribution $\mathcal{N} \left(\hat{\mu}_s(t), \frac{C(\delta, q, t)}{T_s(t)}\right)$ \label{DrawSamplesLine}
   \STATE $\boldsymbol{\theta}^{k}(t) \leftarrow \left( \theta^k_1(t), \ldots, \theta^k_d(t) \right)$
   \STATE $\Tilde{\boldsymbol{\pi}}^{k}(t) \leftarrow \mathrm{Oracle}(\boldsymbol{\theta}^{k}(t))$
   \STATE $\Tilde{\Delta}_t^{k} \leftarrow  {\boldsymbol{\theta}^k(t)}^{\top}\left(\Tilde{\boldsymbol{\pi}}^{k}(t)  - \hat{\boldsymbol{\pi}}(t)\right)$
   \ENDFOR
   \IF{$\forall 1\leq k \leq M(\delta, q, t), \Tilde{\boldsymbol{\pi}}^k(t) = \hat{\boldsymbol{\pi}}(t)$}
    \STATE \textbf{Return:} $\hat{\boldsymbol{\pi}}(t)$
   \ELSE
   \STATE $k^{*}_t \leftarrow \argmax_{k} \Tilde{\Delta}^k_t, \ \Tilde{\boldsymbol{\pi}}(t) \leftarrow \Tilde{\boldsymbol{\pi}}^{k^{*}_{t}}(t)$
   \STATE Pull arm $p_t$ according to (\ref{naive_arm_selection_strategy}) or (\ref{NewArmSelectionStrategy}),and update $T_{p_t}$ and $\hat{\mu}_{p_t}(t)$
   \STATE $t \leftarrow t + 1$
   \ENDIF
   \ENDWHILE
\end{algorithmic}
\end{algorithm}
Then, for any $s\in [d]$, it draws $ M(\delta, q, t) \triangleq \left\lceil\frac{1}{q}(\log 12 |\mathcal{A}|^2 t^2/\delta) \right\rceil$ random samples $\left\{ \theta_{s}^{k} \right\}_{k = 1}^{M(\delta, q, t)}$ independently from a Gaussian distribution $\mathcal{N}\left(\hat{\mu}_s(t), \frac{C(\delta, q, t)}{T_s(t) )} \right)$, and $C(\delta, q, t) \triangleq \frac{4R^2 \log (12 |\mathcal{A}|^2 t^2 / \delta)}{\phi^2(q)}$. Intuitively, $\{\boldsymbol{\theta}^{k}(t)\}_{k = 1}^{M(\delta, q, t)}$ is a set of possible values that the true reward vector $\boldsymbol{\mu}$ can take.
Then, it computes $\Tilde{\boldsymbol{\pi}}^{k}(t) = \mathrm{Oracle}(\boldsymbol{\theta}^k(t))$ for all $k$, where $\boldsymbol{\theta}^{k}(t) = \left( \theta^{k}_{1}(t), \ldots, \theta^{k}_{d}(t) \right) $. We can say that we estimate the true reward gap $\boldsymbol{\mu}^{\top}\left( \Tilde{\boldsymbol{\pi}}^{k}(t)  - \hat{\boldsymbol{\pi}}(t) \right)$ by computing ${\boldsymbol{\theta}^k(t)}^{\top}\left( \Tilde{\boldsymbol{\pi}}^{k}(t)  - \hat{\boldsymbol{\pi}}(t) \right)$ for each $k\in[M(\delta, q, t)]$.
If all the actions $\Tilde{\boldsymbol{\pi}}^{k}(t)$'s are the same as $\hat{\boldsymbol{\pi}}(t)$, we output $\hat{\boldsymbol{\pi}}(t)$ as the best action. 
Otherwise, we focus on $\Tilde{\boldsymbol{\pi}}^{k_t^{*}}(t)$, where $k^{*}_{t} = \argmax_{k \in \left[ M(\delta, q, t) \right]} {\boldsymbol{\theta}^k(t)}^{\top}\left(\Tilde{\boldsymbol{\pi}}^{k}(t)  - \hat{\boldsymbol{\pi}}(t) \right)$. 
We can say that $\Tilde{\boldsymbol{\pi}}^{k_t^{*}}(t)$ is potentially the best action. \par
Then, the most essential question is: ``Which arm should we pull in round $t$, once we obtain the empirically best action $\hat{\boldsymbol{\pi}}(t)$ and a potentially best action $\Tilde{\boldsymbol{\pi}}^{k^{*}} (t)$ ?'' We discuss this below.
\subsubsection{Arm Selection Strategies}
Here, we discuss which arm to pull at round $t$, once we obtain the empirically best action $\hat{\boldsymbol{\pi}}(t)$ and a potentially best action $\Tilde{\boldsymbol{\pi}}^{k_{t}^{*}}(t)$.
For the ordinary CPE-MAB, the arm selection strategy in \citet{WangAndZhu2022} was to pull the following arm:
\begin{equation}\label{naive_arm_selection_strategy}
    p^{\mathrm{naive}}_t = \argmin_{s \in \left\{ s\in[d] \ | \ \hat{\pi}_s(t) \neq \Tilde{\pi}^{k^{*}}_s(t) \right\}} T_s(t).
\end{equation}
Therefore, one candidate of an arm selection strategy is to naively pull the arm defined in~(\ref{naive_arm_selection_strategy}). We call this the \emph{naive arm selection strategy}. \par
Next, we consider another arm selection strategy as follows. We want to pull the arm that is most ``informative'' to discriminate whether $\hat{\boldsymbol{\pi}}(t)$ is a better action than $\Tilde{\boldsymbol{\pi}}^{k_t^{*}}(t)$ or not. In other words, we want to pull the arm that is most ``informative'' to estimate the true gap $\boldsymbol{\mu}^{\top} \left(  \Tilde{\boldsymbol{\pi}}^{k_{t}^{*}}(t) - \hat{\boldsymbol{\pi}}(t) \right) $. If it is less than 0, $\hat{\boldsymbol{\pi}}(t)$ is better, and if it is greater than 0, $\Tilde{\boldsymbol{\pi}}^{k_{t}^{*}}$ is better. To discuss this more quantitatively, let us assume that $\boldsymbol{\theta}^{k_{t}^{*}}(t) \approx \hat{\boldsymbol{\mu}}(t)$. From Hoeffding's inequality \citep{Luo2017}, we obtain the following:
\begin{eqnarray}
    &&\Pr\left[ \left| \left( \boldsymbol{\mu} - \boldsymbol{\theta}^{k_{t}^{*}}(t) \right)^{\top} \left( \Tilde{\boldsymbol{\pi}}^{k_{t}^{*}} (t) - \hat{\boldsymbol{\pi}}(t) \right) \right| \geq \epsilon  \right] \nonumber \\
    &\approx&\Pr\left[ \left| \left( \boldsymbol{\mu} - \hat{\boldsymbol{\mu}}(t) \right)^{\top} \left(  \Tilde{\boldsymbol{\pi}}^{k_{t}^{*}} (t) - \hat{\boldsymbol{\pi}}(t) \right) \right| \geq \epsilon  \right] \nonumber \\
    &\leq& 2 \exp\left( - \frac{ \epsilon^2}{2\displaystyle\sum_{s = 1}^{d} \frac{\left| \Tilde{\pi}^{k_t^{*}}_s(t) - \hat{\pi}_s(t) \right|^2}{T_s(t)} R^2} \right), \label{ArmSelectionStrategyHoeffdingInequality}
\end{eqnarray}
where $\epsilon > 0$. (\ref{ArmSelectionStrategyHoeffdingInequality}) shows that if we make $\sum\limits_{s = 1}^{d} \frac{\left| \Tilde{\pi}^{k_t^{*}}_s(t) - \hat{\pi}_s(t) \right|^2}{T_s(t)}$ small, $\Tilde{\Delta}_t^{^{k_{t}^{*}}} = {\boldsymbol{\theta}^{^{k_{t}^{*}}}(t)}^{\top}\left(\Tilde{\boldsymbol{\pi}}^{^{k_{t}^{*}}}(t)  - \hat{\boldsymbol{\pi}}(t)\right)$ will become close to the true gap $\boldsymbol{\mu}^{\top} \left(  \Tilde{\boldsymbol{\pi}}^{k_{t}^{*}} (t) - \hat{\boldsymbol{\pi}}(t) \right) $. \par
Since we want to estimate the true gap accurately as soon as possible, we pull arm $p^{\mathrm{R}}_t$ that makes $\sum\limits_{s = 1}^{d} \frac{\left| \Tilde{\pi}^{k_t^{*}}_s(t) - \hat{\pi}_s(t) \right|^2}{T_s(t)}$ the smallest, which is defined as follows:
\begin{eqnarray}
    p^{\mathrm{R}}_t 
    &=& \argmin_{e \in [d]} \sum_{s = 1}^{d} \frac{\left| \Tilde{\pi}^{k_t^{*}}_s(t) - \hat{\pi}_s(t) \right|^2}{T_s(t) + \mathbf{1}[s = e]}, \label{p_t^R_equation}
\end{eqnarray}
where $\mathbf{1}[\cdot]$ denotes the indicator function.  
Then, the following proposition holds.
\begin{proposition} \label{R-CPE-MAB_ArmSelectionStrategy_Explicit}
    $p^{\mathrm{R}}_t$ in (\ref{p_t^R_equation}) can be written as follows:
    \begin{eqnarray}
        p^{\mathrm{R}}_t = \argmax_{s \in [d]} \frac{\left| \Tilde{\pi}^{k_t^{*}}_s(t) - \hat{\pi}_s(t) \right|^2}{T_s(t) (T_s(t) + 1)}. \label{NewArmSelectionStrategy} 
    \end{eqnarray}
\end{proposition}
We show the proof in the appendix.
We call pulling the arm defined in (\ref{NewArmSelectionStrategy}) the \emph{R-CPE-MAB arm selection strategy}. (\ref{NewArmSelectionStrategy}) implies that the larger $\left| \Tilde{\pi}^{k_t^{*}}_s(t) - \hat{\pi}_s(t) \right|$ is, the more we need to pull arm $s$. Similar to the discussion in the previous section, this is because if $\left| \Tilde{\pi}^{k_t^{*}}_s(t) - \hat{\pi}_s(t) \right|$ is large, the uncertainty of arm $s$ is amplified largely when we compute $\Tilde{\Delta}_t^{^{k_{t}^{*}}} = {\boldsymbol{\theta}^{k_{t}^{*}} (t)}^{\top}\left(\Tilde{\boldsymbol{\pi}}^{k_{t}^{*}} (t) - \hat{\boldsymbol{\pi}}(t)\right)$. Therefore, we have to pull arm $s$ many times to make the $\frac{C(\delta, q, t)}{T_s(t)}$  small, which is the variance of ${\theta}_s^{k}$, to gain more confidence about the reward of arm $s$. \par
Also, the \emph{R-CPE-MAB arm selection strategy} is equivalent to the \emph{naive arm selection strategy} in CPE-MAB, since when $\mathcal{A}\subset \{0, 1\}^{d}$, 
\begin{eqnarray}
    p^{\mathrm{R}}_t & = &  \argmax_{s \in [d]} \frac{\left| \Tilde{\pi}^{k_t^{*}}_s(t) - \hat{\pi}_s(t) \right|^2}{T_s(t) (T_s(t) + 1)} \nonumber \\
        & = &  \argmax_{s \in \left\{ s\in[d] \ | \ \Tilde{\pi}^{k^{*}_t}_s(t) \neq \hat{\pi}_s(t) \right\} } \frac{1}{T_s(t) (T_s(t) + 1)} \nonumber \\
        & = & \argmin_{s \in \left\{ s\in[d] \ | \ \Tilde{\pi}^{k^{*}_t}_s(t) \neq \hat{\pi}_s(t)  \right\}} T_s(t). \nonumber \\
        & = & p^{\mathrm{naive}}_t.
\end{eqnarray}
Therefore, we can say that the \emph{R-CPE-MAB arm selection strategy} is a generalization of the arm selection strategy in \citet{WangAndZhu2022}.
\subsection{Sample Complexity Upper Bounds of the GenTS-Explore Algorithm} \label{UpperBoundAnalysisSection}
Here, we show sample complexity upper bounds of the GenTS-Explore algorithm when we use the two arm selection strategies: the \emph{naive arm selection strategy} shown in (\ref{naive_arm_selection_strategy}) and the \emph{R-CPE-MAB arm selection strategy} shown in  (\ref{NewArmSelectionStrategy}), respectively. \par
First, in Theorem \ref{TSUpperBoundTheorem_Naive}, we show a sample complexity upper bound of the \emph{naive arm selection strategy}.
\begin{restatable}[]{theorem}{UpperBoundTheorem}\label{TSUpperBoundTheorem_Naive}
    For $q\in[\delta, 0.1]$, with probability at least $1-\delta$, the GenTS-Explore algorithm with the naive arm sampling strategy will output the best action $\boldsymbol{\pi}^{*}$ with sample complexity upper bounded by 
    \begin{equation}
        \mathcal{O}\left( R^2 \mathbf{H}^{\mathrm{N}} \frac{ \left( \log\left( \left|\mathcal{A}\right| \mathbf{H}^{\mathrm{N}} \right) + \log \frac{1}{\delta} \right)^2 }{\log\frac{1}{q}} \right),
    \end{equation}
    where $\mathbf{H}^{\mathrm{N}} = \sum_{s = 1}^d \frac{U_s}{\Delta^2_{(s)}}$ and $U_s = \max_{ \boldsymbol{\pi}' \in \mathcal{A}, \boldsymbol{\pi} \in \{ \boldsymbol{\pi} \in \mathcal{A} \ | \ \pi^{*}_s \neq \pi_s \}}  \frac{1}{\left| \pi^{*}_s - \pi_s \right|^2} \sum_{e = 1}^{d} \left| \pi_e - \pi'_e \right|^2 $. \par
    Specifically, if we choose $q = \delta$, then the complexity upper bound is 
    \begin{equation}
        \mathcal{O}\left( R^2 \mathbf{H}^{\mathrm{N}} \left( \log \frac{1}{\delta} + \log^2 \left( \left|\mathcal{A}\right| \mathbf{H}^{\mathrm{N}}\right) \right) \right). \label{GenTS-ExploreUpperBoundEquation_Naive}
    \end{equation}
\end{restatable} 
Next, in Theorem \ref{TSUpperBoundTheorem_R-CPE-MAB}, we show a sample complexity upper bound of the \emph{R-CPE-MAB arm selection strategy}.
\begin{restatable}[]{theorem}{UpperBoundTheorem}\label{TSUpperBoundTheorem_R-CPE-MAB}
    For $q\in[\delta, 0.1]$, with probability at least $1-\delta$, the GenTS-Explore algorithm with the R-CPE-MAB arm sampling strategy will output the best action $\boldsymbol{\pi}^{*}$ with sample complexity upper bounded by 
    \begin{equation}
        \mathcal{O}\left( R^2 \mathbf{H}^{\mathrm{R}} \frac{ \left( \log\left( \left|\mathcal{A}\right| \mathbf{H}^{\mathrm{R}} \right) + \log \frac{1}{\delta} \right)^2 }{\log\frac{1}{q}} \right),
    \end{equation}
    where $\mathbf{H}^{\mathrm{R}} = \sum_{s = 1}^d \frac{V_s}{\Delta^2_{(s)}}$ and $V_s = \max_{ \boldsymbol{\pi}' \in \mathcal{A}, \boldsymbol{\pi} \in \{ \boldsymbol{\pi} \in \mathcal{A} \ | \ \pi^{*}_s \neq \pi_s \}}  \frac{\left| \pi_s - \pi'_s \right|}{\left| \pi^{*}_s - \pi_s \right|^2} \sum_{e = 1}^{d} \left| \pi_e - \pi'_e \right| $. \par
    Specifically, if we choose $q = \delta$, then the complexity upper bound is 
    \begin{equation}
        \mathcal{O}\left( R^2 \mathbf{H}^{\mathrm{R}} \left( \log \frac{1}{\delta} + \log^2 \left( \left|\mathcal{A}\right| \mathbf{H}^{\mathrm{R}}\right) \right) \right). \label{GenTS-ExploreUpperBoundEquation_R-CPE-MAB}
    \end{equation}
\end{restatable} 
\subsubsection{Comparison to the Lower Bounds}
Let us define $U = \max_{s \in [d]} U_s$ and $V = \max_{s\in[d]} V_s$. Then, the sample complexity upper bound of the naive arm selection strategy is $\mathcal{O}\left( U \mathbf{H} \log\left( \frac{1}{\delta} \right) \right)$ and that of the R-CPE-MAB arm selection strategy is $\mathcal{O}\left( V \mathbf{H} \log\left( \frac{1}{\delta} \right) \right)$. Therefore, regardless of which arm selection strategy is used, the sample complexity upper bound of the GenTS-Explore algorithm matches the lower bound shown in (\ref{LowerBoundTheorem_Explicit_Equation}) up to a problem-dependent constant factor. \par
\subsubsection{Comparison between the Naive and R-CPE-MAB Arm Selection Strategies}
In general, whether the \emph{R-CPE-MAB arm selection strategy} has a tighter upper bound than the \emph{naive arm selection strategy} or not depends on the problem instance. Let us consider one situation in which the R-CPE-MAB arm selection strategy may be a better choice than the naive arm selection strategy. Suppose $\mathcal{A} = \{ (100, 0, 0)^{\top}, (0, 1, 1)^{\top} \}$ and $\boldsymbol{\pi}^{*} = (100, 0, 0)^{\top}$. Then, $U_1 = 1.0002$, $U_2 = 10002$, and $U_3 = 10002$. On the other hand, $V_1 = 1.02$, $V_2 = 102$, and $V_3 = 102$. We can see that $U_2$ and $U_3$ are extremely larger than $V_2$ and $V_3$, respectively, and therefore $\mathbf{H}^{\mathrm{R}}$ is much smaller than $\mathbf{H}^{\mathrm{N}}$. Eventually, the sample complexity upper bound of the naive arm selection strategy will be looser than that of the R-CPE-MAB arm selection strategy. \par
\subsubsection{Comparison with Existing Works in the Ordinary CPE-MAB}
In the ordinary CPE-MAB, where $\mathcal{A}\subseteq \{0, 1\}^{d}$, a key notion called \emph{width} appears in the upper bound of some existing algorithms \citep{SChen2014,WangAndZhu2022}, which is defined as follows:
\begin{eqnarray}
    \mathrm{width} = \max_{\boldsymbol{\pi}, \boldsymbol{\pi}' \in \mathcal{A}} \sum_{s = 1}^{d} \left| \pi_s - \pi'_s \right|.
\end{eqnarray}
The following proposition implies that both $U$ and $V$ can be seen as generalizations of the notion \emph{width}.
\begin{proposition} \label{U_V_proposition}
    Let $U = \max_{s \in [d]} U_s$ and $V = \max_{s\in[d]} V_s$. In the ordinary CPE-MAB, where $\mathcal{A} \subseteq \{0, 1\}^{d}$, we have
    \begin{eqnarray}
        U = V = \mathrm{width}.
    \end{eqnarray}
\end{proposition}
Next, recall that the GenTS-Explore algorithm is equivalent to the TS-Explore algorithm in the ordinary CPE-MAB, regardless of which arm selection strategy is used. 
Proposition \ref{TighterResultThanWangAndZhu} shows that our upper bound (\ref{GenTS-ExploreUpperBoundEquation_Naive}) and (\ref{GenTS-ExploreUpperBoundEquation_R-CPE-MAB}) are both tighter than that shown in \citet{WangAndZhu2022}, which is $\mathcal{O} \left( \mathrm{width} \sum_{s = 1}^{d} \frac{1}{\Delta^2_s} \right)$.
\begin{proposition} \label{TighterResultThanWangAndZhu}
    In the ordinary CPE-MAB, where $\mathcal{A} \subseteq \{0, 1\}^{d}$, we have
    \begin{eqnarray}
        \mathbf{H}^{\mathrm{N}} = \sum_{s = 1}^{d} \frac{U_s}{\Delta^2_s} \leq \mathrm{width} \sum_{s = 1}^{d} \frac{1}{\Delta^{2}_{s}}, \label{H^N_result}
    \end{eqnarray}
    and 
    \begin{eqnarray}
        \mathbf{H}^{\mathrm{R}} = \sum_{s = 1}^{d} \frac{V_s}{\Delta^2_s} \leq \mathrm{width} \sum_{s = 1}^{d} \frac{1}{\Delta^{2}_{s}}. \label{H^R_result}
    \end{eqnarray}
\end{proposition}
\subsubsection{Comparison with Existing Works in the Transductive Bandits Literature}
Even though it is impractical to apply the RAGE algorithm \citep{Fiez2019} and the Peace algorithm \citep{Fiez2019} to the R-CPE-MAB when the action class $\mathcal{A}$ is exponentially large in $d$, it is still beneficial to theoretically compare the GenTS-Explore algorithm with them for the case where the action class $\mathcal{A}$ is polynomial in $d$. 
Firstly, from the lower bound of \citet{Fiez2019}, we have $\mathbf{H}^{R}, \mathbf{H}^{N} \geq \rho_{*}$. This means the upper bound shown in (\ref{GenTS-ExploreUpperBoundEquation_Naive}) and (\ref{GenTS-ExploreUpperBoundEquation_R-CPE-MAB}) are at least $\rho_{*} \left( \log \left( \frac{1}{\delta} + \log^{2} \left( |\mathcal{A}| \rho_{*} \right) \right) \right)$.  \par
The results of \citet{Fiez2019} shows that the upper bound of the RAGE is $\rho_{*} \left( \log \left( \frac{1}{\delta} \right) + \log \left( |\mathcal{A}| \right) \right) \log \left( \frac{1}{\Delta_{\mathrm{min}}} \right)$, where $\Delta_{\mathrm{min}} = \min_{\mathcal{A} \setminus \{ \boldsymbol{\pi}^{*} \}} \boldsymbol{\mu}^{\top} \left( \boldsymbol{\pi}^{*} - \boldsymbol{\pi} \right)$. The GenTS-Explore pay $\log^{2} \left( |\mathcal{A}| \right)$ and the RAGE pay $\log\left( |\mathcal{A}| \right) \log \left( \frac{1}{\Delta_{\mathrm{min}}} \right)$. In general, the size of the action set is exponentially large in $d$, and therefore the RAGE algorithm has a tighter upper bound. \par
The Peace algorithm \citep{KatzSamuels2020} has a sample complexity upper bound of 
\begin{eqnarray}
    \left(\gamma_{*} + \rho_{*} \log\left( \frac{1}{\delta}\right) \right) \log \left( \frac{1}{\Delta_{\mathrm{min}}} \right), \label{PeaceUpperBound}
\end{eqnarray}
where 
\begin{eqnarray}
    \gamma_{*} = \inf_{\boldsymbol{\lambda} \in \Pi_{d}} \mathbb{E}_{\boldsymbol{\eta} \sim \mathcal{N}\left(0, I \right) } \left[ \sup_{\boldsymbol{\pi} \in \mathcal{A} \setminus \left\{ \boldsymbol{\pi}^{*} \right\}} \frac{\left( \boldsymbol{\pi}^{*} - \boldsymbol{\pi} \right)^{\top} \boldsymbol{A}\left( \lambda \right)^{-1/2} \boldsymbol{\eta}}{ \boldsymbol{\mu}^{\top} \left( \boldsymbol{\pi}^{*} - \boldsymbol{\pi} \right) } \right]. \nonumber
\end{eqnarray}
Here, $\boldsymbol{A}\left(\boldsymbol{\lambda}\right)$ is a diagonal matrix whose $(i, i)$ element is $\lambda_{i}$. They showed that the upper bound in (\ref{PeaceUpperBound}) matches the lower bound up to $\log\left( \frac{1}{\Delta_{\mathrm{min}}} \right)$, which gets rid of the $\log\left( |\mathcal{A}| \right)$ term. It is an interesting future work whether we can get rid of the $\log\left( |\mathcal{A}| \right)$ even for the R-CPE-MAB with exponentially large $\mathcal{A}$.

\section{Experiment}\label{ExperimentSection}
In this section, we experimentally compare the two arm selection strategies, the \emph{naive arm selection strategy} and the \emph{R-CPE-MAB arm selection strategy}. 
\subsection{The Knapsack Problem}
Here, we consider the knapsack problem \citep{dantzig2007}, where the action set $\mathcal{A}$ is exponentially large in $d$ in general. \par
In the knapsack problem, we have $d$ items. Each item $s\in[d]$ has a weight $w_s$ and value $v_s$. Also, there is a knapsack whose capacity is $W$ in which we put items. Our goal is to maximize the total value of the knapsack not letting the total weight of the items exceed the capacity of the knapsack. Formally, the optimization problem is given as follows:
\begin{equation*}
\begin{array}{ll@{}ll}
\text{maximize}_{\boldsymbol{\boldsymbol{\pi}}\in\mathcal{A}}  & \sum_{s = 1}^{d}v_{s}\pi_s  &\\ \\
\text{subject to}& \sum_{s = 1}^{d}\pi_s w_s \leq W, &
\end{array}
\end{equation*}
where $\pi_s$ denotes the number of item $s$ in the knapsack. Here, the weight of each item is known, but the value is unknown, and therefore has to be estimated. In each time step, the player chooses an item $s$ and gets an observation of value $r_s$, which can be regarded as a random variable from an unknown distribution with mean $v_s$. \par
For our experiment, we generated the weight of each item uniformly from $\{1, 2, \ldots, 50 \}$. For each item $s$, we generated $v_s$ as $v_s = w_s \times (1 + x)$, where $x$ is a sample from $\mathcal{N}(0, {0.1}^2)$. We set the capacity of the knapsack at $W =50$. Each time we chose an item $s$, we observed a value $v_s + x$ where $x$ is a noise from $\mathcal{N}(0, {0.1}^2)$. We set $R=0.1$. We show the result in Figure \ref{knapsack_problem_experiment}. \par
We can say that the R-CPE-MAB arm selection strategy performs better than the naive arm selection strategy since the former needs fewer rounds until termination. In some cases, the sample complexity of the R-CPE-MAB arm selection strategy is only 1/3 to 1/2 that of the naive arm selection strategy.
\begin{figure}
    \centering
    \includegraphics[width = \linewidth]{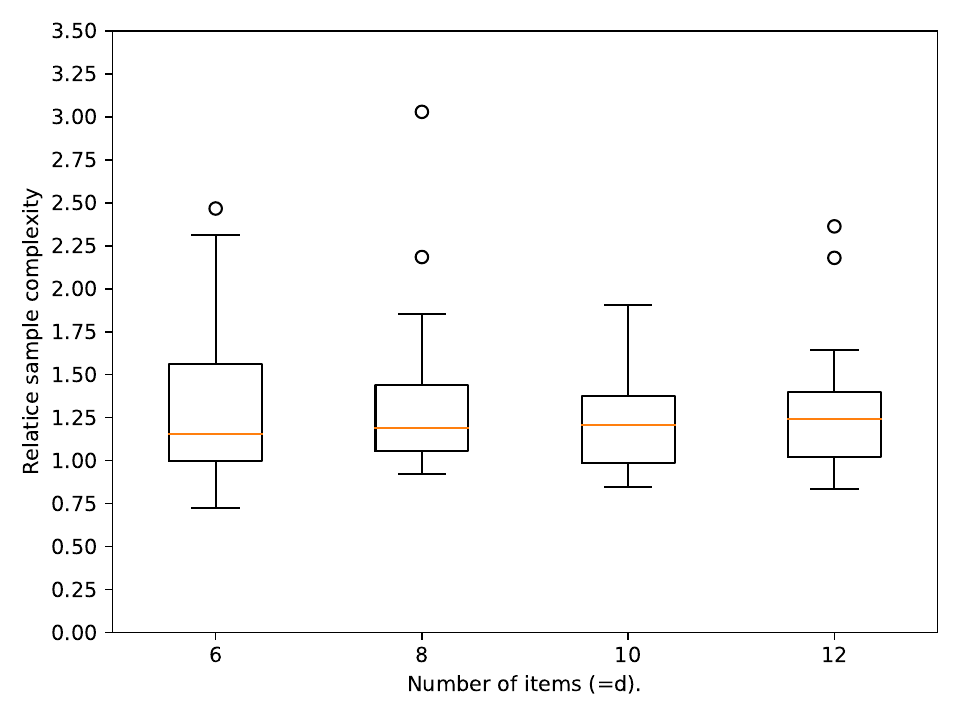}
    \caption{Comparison of the naive arm selection strategy and the R-CPE-MAB arm selection strategy. The vertical axis indicates the number of rounds the former strategy took to find the best action normalized by the number of rounds the latter strategy took to find the best action. The horizontal axis indicates the number of items $d$. We ran experiments 30 times for each setting.}
    \label{knapsack_problem_experiment}
\end{figure}
\subsection{The Production Planning Problem}
Here, we consider the production planning problem \citep{Pochet2010}. In the production planning problem, there are $m$ materials, and these materials can be mixed to make one of $d$ different products. We have a matrix $\boldsymbol{M} \in \mathbb{R}^{m \times d}$, where $M_{s}$ represents how much material $i \in [m]$ is needed to make product $s \in [d]$. 
Also, we are given vectors $\boldsymbol{v}^{\mathrm{max}} \in \mathbb{R}^{m}$ and $\boldsymbol{\mu} \in \mathbb{R}^{d}$.
Then, formally, the optimization problem is given as follows:
\begin{equation*}
\begin{array}{ll@{}ll}
\text{maximize}_{\boldsymbol{\boldsymbol{\pi}}\in\mathcal{A}}  & \boldsymbol{\mu}^{\top} \boldsymbol{\pi}  &\\ \\
\text{subject to}& \boldsymbol{M} \boldsymbol{\pi} \leq \boldsymbol{v}^{\mathrm{max}} , &
\end{array}
\end{equation*}
where the inequality is an element-wise comparison. 
Intuitively, we want to obtain the optimal vector $\boldsymbol{\pi}^{*}$ that maximizes the total profit without using more material $i$ than $v^{\mathrm{max}}_i$ for each $i \in [m]$, where $\pi^{*}_s$ represents how much product $s$ is produced. \par
Here, we assume that $\boldsymbol{M}$ and $\boldsymbol{v}^{\mathrm{max}}$ are known, but $\boldsymbol{\mu}$ is unknown, and therefore has to be estimated. In each time step, the player chooses a product $s$ and gets an observation of value $r_s$, which can be regarded as a random variable from an unknown distribution with mean $\mu_s$.\par
For our experiment, we have three materials, i.e., $m = 3$. We set $\boldsymbol{v}^{\mathrm{max}} = (30, 30, 30)^{\top}$. Also, we generated every element in $M$ uniformly from $\{1, 2, 3, 4\}$. For each product $s$, we generated $\mu_s$ as $\mu_s = \sum_{i = 1}^{m} M_{is} + x$, where $x$ is a random sample from $\mathcal{N}(0,1)$. Each time we chose a product $s$, we observed a value $\mu_s + x$ where $x$ is a noise from $\mathcal{N}\left(0, {0.1}^2\right)$. We set $R = 0.1$. We show the result in Figure \ref{production_planning_experiment}. Again, we can see that the R-CPE-MAB arm selection strategy performs better than the naive arm selection strategy since the former needs fewer rounds until termination.

\begin{figure}
    \centering
    \includegraphics[width = \linewidth]{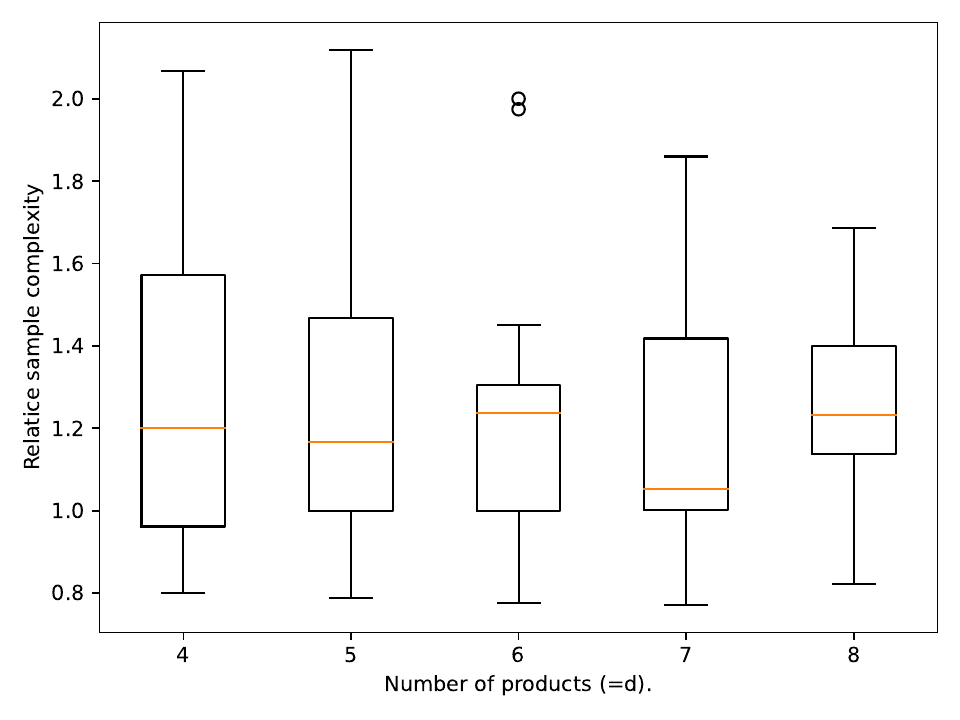}
    \caption{ The numbers show the mean and standard deviation of the number of rounds the naive arm selection strategy took to find the best action normalized by the number of rounds the R-CPE-MAB arm selection strategy took to find the best action over 15 runs. }
    \label{production_planning_experiment}
\end{figure}

\section{Conclusion}
In this study, we studied the R-CPE-MAB. We showed novel lower bounds for R-CPE-MAB by generalizing key quantities in the ordinary CPE-MAB literature. Then, we introduced an algorithm named the GenTS-Explore algorithm, which can identify the best action in R-CPE-MAB even when the size of the action set is exponentially large in $d$. We showed a sample complexity upper bound of it, and showed that it matches the sample complexity lower bound up to a problem-dependent constant factor. Finally, we experimentally showed that the GenTS-Explore algorithm can identify the best action even if the action set is exponentially large in~$d$. 

\section*{Acknowledgement}
We thank Dr. Kevin Jamieson for his very helpful advice and comments on existing studies.
SN was supported by JST SPRING, Grant Number JPMJSP2108.

\bibliography{aaai24}

\clearpage
\appendix

\section{Proof of Theorem \ref{LowerBoundTheorem_Explicit}}
For the reader's convenience, we restate Theorem \ref{LowerBoundTheorem_Explicit}.
\LowerBoundTheoremExplicit*
Before stating our proof, we first introduce two technical lemmas. The first lemma is the well-known Kolmogrov's inequality.
\begin{lemma}[Kolmogrov's inequality[\citet{SChen2014}, Lemma 14]] \label{KokomogorovInequality}
    Let $Z_1, \ldots, Z_n$ be independent zero-mean random variables with $\mathrm{Var}\left[ Z_k \right] \leq +\infty$ for all $k\in[n]$. Then, for any $\lambda > 0$, 
    \begin{equation}
        \Pr\left[ \max_{1 \leq k \leq n} |S_k| \geq \lambda \right] \leq \frac{1}{\lambda^2} \sum_{i = 1}^{n} \mathrm{Var}[Z_k],
    \end{equation}
    where $S_k = X_1 + \cdots + X_k$.
\end{lemma}
The second technical lemma shows that the joint likelihood of Gaussian distributions on a sequence of variables does not change  much when the mean of the distribution shifts by a sufficiently small value.
\begin{lemma}[\citet{SChen2014}] \label{Lemma15inSChen2014}
    Fix some $d \in \mathbb{R}$ and $\theta \in (0, 1)$. Define $t = \frac{1}{4d^2}\log(\frac{1}{\theta})$. Let $T$ be an integer less or equal to $4t$, $s_1, \ldots, s_T$ be any sequence, and $X_1, \ldots, X_T$ be $T$ real numbers which satisfy the following:
    \begin{equation}
        \left| \sum_{i = 1}^{T} X_i - \sum_{i = 1}^{T} s_i \right| \leq \sqrt{t \log\left(\frac{1}{\theta}\right)}.
    \end{equation}
    Then, we have
    \begin{equation}
        \prod_{i = 1}^{T} \frac{\mathcal{N}(X_i | s_i + d, 1)}{\mathcal{N}(X_i | s_i, 1)} \geq\theta,
    \end{equation}
    where we let $\mathcal{N}(x, \mu, \sigma^2) = \frac{1}{\sigma\sqrt{2\pi}}\exp\left( -\frac{\left( x - \mu \right)^2}{2\sigma^2} \right)$ denote the probability density function of normal distribution with mean $\mu$ and variance $\sigma^2$.
\end{lemma}
Now, we prove Theorem \ref{LowerBoundTheorem_Explicit}.
\begin{proof}[Proof of Theorem \ref{LowerBoundTheorem_Explicit}]
    Fix $\delta > 0$, $\boldsymbol{\mu} = \left( \mu_1, \ldots, \mu_d \right)^{\top}$, and a $\delta$-correct algorithm $\mathbb{A}$. For each $s \in [d]$, assume that the reward distribution is given by $\phi_s = \mathcal{N}\left( \mu_s, 1 \right)$. For any $s \in [d]$, the number of trials of arm $s$ is lower-bounded by 
    \begin{equation}
        \mathbb{E}\left[ T_s \right] \geq \frac{1}{16\Delta_{(s)}^2} \log \left( \frac{1}{4\delta} \right). \label{Eq83inSChen}
    \end{equation}
    Note that the theorem follows immediately by summing up (\ref{Eq83inSChen}) for all $s\in[d]$.\par
    Fix any $s\in [d]$. We define $\theta = 4\delta$ and $t_{s}^{*} = \frac{1}{16\Delta_{(s)}^2} \log \left( \frac{1}{\theta} \right)$. We prove (\ref{Eq83inSChen}) by contradiction. Therefore, we assume $\mathbb{E}[T_s] < t_{s}^{*}$ in the rest of the proof. \par
    \textbf{Step (1): An alternative hypothesis.} We consider two hypothesis $H_0$ and $H_1$. Under hypothesis $H_0$, all reward distributions are the same with our assumption in the theorem as follows:
    \begin{equation}
        H_0: \phi_s = \mathcal{N}\left( \mu_s, 1 \right) \quad \text{for all $s\in[d]$}.
    \end{equation}
    On the other hand, under hypothesis $H_1$, we change the means of reward distributions such that
    \begin{equation}
        H_1: \phi_s = 
      \begin{cases}
        \mathcal{N}\left( \mu_s - 2\Delta_{(s)}, 1 \right) & \text{if $\pi^{*}_{s} > \pi^{(s)}_{s}$,} \\
        \mathcal{N}\left( \mu_s + 2\Delta_{(s)}, 1 \right) & \text{if $\pi^{*}_{s} < \pi^{(s)}_{s}$,}
      \end{cases}  
    \end{equation}
    and  $\phi_l = \mathcal{N}(\mu_l, 1)$ for all $l \neq s$. \par
    For $h \in \{ 0, 1 \}$, we use $\mathbb{E}_{h}$ and $\mathbb{P}_{h}$ to denote the expectation and probability, respectively, under the hypothesis $H_h$. \par
    Now, we claim that $\boldsymbol{\pi}^{*}$ is no longer the optimal action under hypothesis $H_1$. Let $\boldsymbol{\mu}_{0}$ and $\boldsymbol{\mu}_{1}$ be the expected reward vectors under $H_0$ and $H_1$, respectively. We have
    \begin{eqnarray}
        && \boldsymbol{\mu}_1^{\top} \boldsymbol{\pi}^{*} - \boldsymbol{\mu}_1^{\top} \boldsymbol{\pi}^{(s)} \nonumber \\
        & = & \boldsymbol{\mu}_{0}^{\top} \boldsymbol{\pi}^{*} - \boldsymbol{\mu}_{0}^{\top} \boldsymbol{\pi}^{(s)} - 2\left( \boldsymbol{\mu}_{0}^{\top} \boldsymbol{\pi}^{*} - \boldsymbol{\mu}_{0}^{\top} \boldsymbol{\pi}^{(s)} \right) \nonumber \\
        & = & - \left( \boldsymbol{\mu}_{0}^{\top} \boldsymbol{\pi}^{*} - \boldsymbol{\mu}_{0}^{\top} \boldsymbol{\pi}^{(s)} \right) \nonumber \\
        & < & 0. \nonumber
    \end{eqnarray}
    This means that under $H_1$, action $\boldsymbol{\pi}^{*}$ is not the best action. \par
    \textbf{Step (2): Three random events.} Let $X_1, \ldots, X_{T_s}$ denote the sequence of reward outcomes of arm $s$. Now, we define three random events $\mathcal{A}$, $\mathcal{B}$, and $\mathcal{C}$ as follows:
    \begin{eqnarray}
        \mathcal{A} &=& \left\{ T_s \leq 4t_s^{*} \right\}, \nonumber\\
        \mathcal{B} &=& \left\{ \boldsymbol{\pi}_{\mathrm{out}} = \boldsymbol{\pi}^{*} \right\}, \nonumber\\ 
        \mathcal{C} &=& \left\{ \max_{1 \leq t \leq 4t^{*}_s} \left| \sum_{i = 1}^{t} X_i - t\cdot \mu_s \right| < \sqrt{t^{*}_s \log \left( \frac{1}{\delta} \right)} \right\}. \nonumber
    \end{eqnarray}
    Now, we bound the probability of these events under hypothesis $H_0$. First, we show that $\Pr\left[ \mathcal{A} \right] \geq 3/4$. This can be proven by Markov's inequality as follows:
    \begin{equation}
        {\Pr}_0\left[ T_s > 4t^{*}_s \right] \leq \frac{\mathbb{E}_0\left[ T_s \right]}{4t^{*}_s} \leq \frac{t^{*}_s}{4t^{*}_s} = \frac{1}{4}. \nonumber
    \end{equation}
    We now show that $\{\Pr\}_0\left[ \mathcal{C} \right]$. Notice that $\left\{ X_t - \mu_s \right\}_{t = 1, \ldots}$ is a sequence of zero-mean independent random variables under $H_0$. Define $K_t = \sum_{i = 1}^{t}X_i$. Then, by Kolomogorov's inequality (Lemma \ref{KokomogorovInequality}), we have 
    \begin{eqnarray}
        &&{\Pr}_{0}\left[ \max_{1 \leq t \leq 4t^{*}_s} \left| K_t - t\cdot\mu_s \right| \geq \sqrt{t^{*}_s \log \left( \frac{1}{\theta} \right)} \right] \nonumber \\
        & \leq & \frac{{\mathbb{E}_0\left[ \left( K_{4t^*_s} - 4\mu_s t^{*}_s \right)^2 \right]}}{t^{*}_s \log\left( \frac{1}{\theta} \right)} \nonumber \\
        & \overset{(a)}{=} & \frac{4t^{*}_s}{t^{*}_s \log\left( \frac{1}{\theta} \right) } \nonumber \\
        & \overset{(b)}{<} & \frac{1}{4},
    \end{eqnarray}
    where (a) follows from the fact that the variance of $\phi_s$ is 1 and therefore $\mathbb{E}_{0}\left[ \left( K_{4t^{*}_s} - 4\mu_s t^{*}_s \right)^2 \right] = 4t^{*}_s$, and (b) follows from  $\theta < e^{-16}$. \par
    Since $\mathbb{A}$ is a $\delta$-correct algorithm, where $\delta < \frac{e^{-16}}{4} < \frac{1}{4}$, we have ${\Pr}_{0}\left[ \mathcal{B} \right] \geq \frac{3}{4}$. Define random event $\mathcal{S} = \mathcal{A}\cap \mathcal{B}\cap \mathcal{C}$. Then, by union bound, we have ${\Pr}_0\left[ \mathcal{S} \right] \geq \frac{1}{4}$.\par
    \textbf{Step(3): The loss of likelihood.} Now, we claim that, under the assumption that $\mathbb{E}_{0}\left[ T_s \right] < t^{*}_s$, one has ${\Pr}_1\left[ \mathcal{B} \right] \geq \delta$. Let $W$ be the history of the sampling process until the algorithm stops (including the sequence of arms chosen at each time and the sequence of observed outcomes). Define the likelihood function $L_l$ as 
    \begin{equation}
        L_l(w) = p_l(W = w),
    \end{equation}
    where $p_l$ is the probability density function of histories under hypothesis $H_l$. \par
    Now, assume that the event $\mathcal{S}$ occurred. We will bound the likelihood ratio $L_1(W)/L_0(W)$ under this assumption. Since $H_1$ and $H_0$ only differs on the reward distribution of arm $s$, we have
    \begin{equation} \label{likelihoodratio}
        \frac{L_1(W)}{L_0(W)} = \prod_{i = 1}^{T_s} \frac{\mathcal{N} (X_i | \mu_{1, s}, 1) }{\mathcal{N}(X_i | \mu_{0, s}, 1)},
    \end{equation}
    where $\mu_{0, s}$ and $\mu_{1, s}$ denotes the $s$-th element of $\boldsymbol{\mu}_{0}$ and $\boldsymbol{\mu}_{1}$, respectively. By the definition of $H_1$ and $H_0$, we see that $\mu_{1, s} = \mu_{0, s} \pm 2\Delta_{(s)}$, where the sign depends on whether $\pi_{s}^{*}$ is larger than $\pi^{(s)}_s$ or not. Therefore, when event $\mathcal{S}$ occurs, we can apply Lemma \ref{Lemma15inSChen2014} by setting $d = \mu_{1, s} - \mu_{0, s} = \pm 2\Delta_{(s)}$, $T = T_s$ and $s_i = \mu_{0, s}$ for all $i$. Hence, by Lemma \ref{Lemma15inSChen2014} and (\ref{likelihoodratio}), we have
    \begin{equation}
        \frac{L_1(W)}{L_0(W)} \geq \theta = 4\delta
    \end{equation}
    holds if event $\mathcal{S}$ occurs. \par
    Then, we have 
    \begin{equation}
        \frac{L_1(W)}{L_0(W)}\mathbf{1}[\mathcal{S}] \geq 4\delta\mathbf{1}[\mathcal{S}]
    \end{equation}
    holds regardless the occurence of event $\mathcal{S}$. Here, recall that $\mathbf{1}[\cdot]$ denotes the indicator function, i.e., $\mathbf{1}[\mathcal{S}] = 1$ only if $\mathcal{S}$ occurs and otherwise $\mathbf{1}[\mathcal{S}] = 0$.\par
    We can obtain
    \begin{eqnarray}
        {\Pr}_1\left[ \mathcal{B} \right] 
        & \geq & {\Pr}_1 \left[ \mathcal{S} \right] =  {\mathbb{E}}_1 \left[ \mathbf{1}[\mathcal{S}] \right] \\
        & = & {\mathbb{E}}_{0} \left[ \frac{L_{1}(W)}{L_{0}(W)} \mathbf{1}[\mathcal{S}] \right] \\
        & \geq & 4 \delta \mathbb{E}_{0}\left[\mathbf{1}[\mathcal{S}]  \right] \\
        & = & 4\delta{\Pr}_{0}\left[ \mathcal{S} \right] \geq \delta.        
    \end{eqnarray}
    Now, we have proven that, if $\mathbb{E}\left[ T_S \right] < t^{*}_s$, then $\Pr\left[ \mathcal{B} \right] \geq \delta$. This means that, if $\mathbb{E}\left[ T_S \right] < t^{*}_s$, algorithm $\mathbb{A}$ will choose $\boldsymbol{\pi}^{*}$ as the final output $\boldsymbol{\pi}_{\mathrm{out}}$ with probability at least $\delta$, under hypothesis $H_1$. However, under $H_1$, we have shown that $\boldsymbol{\pi}^{*}$ is not the best action. Therefore, algorithm $\mathbb{A}$ has a probability of error at least $\delta$ under $H_1$. This contradicts to the assumption that algorithm $\mathbb{A}$ is a $\delta$-correct algorithm. Hence, we must have $\mathbb{E}_{0}\left[ T_s \right] \geq t^{*}_s = \frac{1}{16\Delta_{(s)}^2} \log \left( \frac{1}{4\delta} \right)$
\end{proof}

\section{Proof of Theorem \ref{LowerBoundTheoremImplicit}} \label{LowerBoundTheoremImplicitProof}
Here, we prove Theorem \ref{LowerBoundTheoremImplicit}. Before proving the theorem, let us denote the Kullback-Leibler divergence from the distribution of arm $a_2$ to that of arm $a_1$ by $\mathrm{KL}(a_1, a_2)$. We call an arm a \emph{Gaussian arm} if its reward follows a Gaussian distribution with unit variance. For two Gaussian arms $a_1$ and $a_2$ with means $\mu_1$ and $\mu_2$, respectively, it holds that 
\begin{equation}
    \mathrm{KL}(a_1, a_2) = \frac{1}{2}(\mu_1 - \mu_2)^{2}
\end{equation}
Moreover, let us denote the binary relative entropy function by $d(x, y) = x \log (\frac{x}{y}) + (1 - x)\log(\frac{1 - x}{1 - y})$. \par
Next, we show the ``Change of Distribution'' lemma formulated by \citet{Kaufmann2016}, which is useful for proving our lower bound.
\begin{lemma} \label{Change_of_Distribution}
    Let $\mathbb{A}$ be an algorithm that runs on $d$ arms, and let $\mathcal{C} = (a_1, \ldots, a_d)$ and $\mathcal{C}' = (a'_1, \ldots, a'_d)$ be two instances of $d$ arms. Let random variable $\tau_s$ denote the number of samples taken from the $s$-th arm. Let $\mathcal{E}$ denote the event that algorithm $\mathbb{A}$ returns $\boldsymbol{\pi}^{*}$ as the optimal action. Then, we have
    \begin{equation}
        \sum_{s = 1}^{d} \mathbb{E}_{\mathbb{A}, \mathcal{C}}[\tau_s] \mathrm{KL}(a_s, a'_s) \geq d\left( \Pr_{\mathbb{A}, \mathcal{C}}\left[ \mathcal{E} \right], \Pr_{\mathbb{A}, \mathcal{C}'}\left[ \mathcal{E} \right] \right).
    \end{equation}
\end{lemma}
Now, we prove Theorem \ref{LowerBoundTheoremImplicit}. We define $U_{\boldsymbol{\pi}, \boldsymbol{\pi}'} = \left\{ s\in[d] \ | \ \pi_s > \pi'_s \right\}$. We restate Theorem \ref{LowerBoundTheoremImplicit}.
\LowerBoundTheoremImplicit*
\begin{proof}
    Fix $\delta \in (0, 0.1)$, action set $\mathcal{A}$, and a $\delta$-correct algorithm $\mathbb{A}$. Let $n_s$ be the expected number of samples drawn from the $s$-th arm when $\mathbb{A}$ runs on instance $\mathcal{C}$. Let $\alpha = d(1 - \delta, \delta)/2$ and $\tau_s/\alpha$. It suffices to show that $\tau = \{\tau_1, \ldots, \tau_d\}$ is a feasible solution for the optimization problem (\ref{Low(A)OptProblem}), as it follows that 
    \begin{equation}
        \sum_{s = 1}^{d} n_s = \alpha \sum_{s = 1}^{d} \tau_s \geq \alpha \mathrm{Low}(\mathcal{A}) = \Omega\left( \mathrm{Low}(\mathcal{A}) \log \frac{1}{\delta} \right).
    \end{equation}
    Here, the last equation holds since for all $\delta \in (0, 0.1)$, 
    \begin{equation}
        d(1 - \delta, \delta) = (1 - 2\delta)\log \frac{1 - \delta}{\delta} \geq 0.8 \log \frac{1}{\sqrt{\delta}} = 0.4 \log \frac{1}{\delta}.
    \end{equation}
    To show that $\tau$ is a feasible solution, we fix $\boldsymbol{\pi} \in \mathcal{A}$. Let $s \in U_{\boldsymbol{\pi}^{*}, \boldsymbol{\pi}} \cup U_{\boldsymbol{\pi} , \boldsymbol{\pi}^{*}}$ and  $\frac{\boldsymbol{\mu}^{\top}(\boldsymbol{\pi}^{*} - \boldsymbol{\pi})}{|\pi^{*}_s - \pi_s|} = \frac{c_s}{n_s}$, where 
    \begin{equation}
        c_s = \frac{2\boldsymbol{\mu}^{\top}(\boldsymbol{\pi}^{*} - \boldsymbol{\pi}) \cdot \left| \pi^{*}_s - \pi_s \right|}{\sum_{s = 1}^{d} \frac{\left| \pi^{*}_s - \pi_s \right|^2}{n_s}}.
    \end{equation}
    We consider the following alternative instance $\mathcal{I}'$: the mean of each arm $s$ in $U_{\boldsymbol{\pi}^{*}, \boldsymbol{\pi}}$ is decreased by $\frac{\boldsymbol{\mu}^{\top}(\boldsymbol{\pi}^{*} - \boldsymbol{\pi})}{|\pi^{*}_s - \pi_s|}$, while the mean of each arm $s$ in $U_{\boldsymbol{\pi}, \boldsymbol{\pi}^{*}}$ is increased by $\frac{\boldsymbol{\mu}^{\top}(\boldsymbol{\pi}^{*} - \boldsymbol{\pi})}{|\pi^{*}_s - \pi_s|}$, and the action set $\mathcal{A}$ is the same as $\mathcal{I}$. Note that, in $\mathcal{I}'$,
    \begin{eqnarray}
    &&\begin{split}
        \biggl( \boldsymbol{\mu}^{\top}\boldsymbol{\pi}^{*} &+ \sum_{s \in U_{\boldsymbol{\pi}^{*}, \boldsymbol{\pi}}} \left( - \frac{\boldsymbol{\mu}^{\top}(\boldsymbol{\pi}^{*} - \boldsymbol{\pi})}{|\pi^{*}_s  - \pi_s|} \right) \left(\pi^{*}_s -\pi_s\right) \biggr) \\
        &- \biggl( \boldsymbol{\mu}^{\top}\boldsymbol{\pi}^{*} + \sum_{s \in U_{\boldsymbol{\pi}, \boldsymbol{\pi}^{*}}} \frac{\boldsymbol{\mu}^{\top}(\boldsymbol{\pi}^{*} - \boldsymbol{\pi})}{|\pi^{*}_s - \pi_s|} \left(\pi_s -\pi^{*}_s \right) \biggr)
    \end{split} \nonumber \\
    &=& - \boldsymbol{\mu}^{\top}(\boldsymbol{\pi}^{*} - \boldsymbol{\pi}) < 0. \nonumber
    \end{eqnarray}
    In other words, $\boldsymbol{\pi}^{*}$ is no longer the optimal action in $\mathcal{I}'$. \par
    Let $\mathcal{E}$ denote the event that algorithm $\mathbb{A}$ returns $\boldsymbol{\pi}^{*}$ as the optimal action. Note that since $\mathbb{A}$ is $\delta$-correct algorithm, $\Pr_{\mathcal{I}}\left[ \mathcal{E} \right] \geq 1 - \delta$ and $\Pr\left[ \mathcal{E} \right] \leq \delta$. Therefore, by Lemma \ref{Change_of_Distribution}, 
    \begin{eqnarray}
        &&\sum_{s \in U_{\boldsymbol{\pi}^{*}, \boldsymbol{\pi}} \cup U_{\boldsymbol{\pi} , \boldsymbol{\pi}^{*}}} n_s\cdot \frac{1}{2} \left( \frac{\boldsymbol{\mu}^{\top}(\boldsymbol{\pi}^{*} - \boldsymbol{\pi})}{|\pi^{*}_s - \pi_s|}\right)^2  \nonumber \\
        &\geq& d\left( \Pr_{\mathcal{I}}\left[ \mathcal{E} \right], \Pr_{\mathcal{I}}\mathcal{E} \right) \nonumber \\
        &\geq& d(1 - \delta, \delta). \nonumber
    \end{eqnarray}
    We have
    \begin{eqnarray}
        &&2d(1 - \delta, \delta) \nonumber \\
        &\leq& \sum_{s = 1}^{d} n_s\cdot \frac{c_s^2}{n^2_s} \nonumber \\
        &=& \frac{4\left( \boldsymbol{\mu}^{\top} \left( \boldsymbol{\pi}^{*} - \boldsymbol{\pi} \right)  \right)^2}{\left(\sum_{s = 1}^{d} \frac{\left| \pi^{*}_s - \pi_s \right|^2}{n_s}\right)^2} \sum_{s = 1}^{d} \frac{\left| \pi^{*}_s - \pi_s \right|^2}{n_s} \nonumber\\
        &=& \frac{4\left( \boldsymbol{\mu}^{\top} \left( \boldsymbol{\pi}^{*} - \boldsymbol{\pi} \right)  \right)^2}{ \sum_{s = 1}^{d} \frac{\left| \pi^{*}_s - \pi_s \right|^2}{n_s}},
    \end{eqnarray}
    and it follows that
    \begin{equation}
        \sum_{s = 1}^{d} \frac{\left| \pi^{*}_s - \pi_s \right|^2}{n_s} \leq  \left( \boldsymbol{\mu}^{\top} \left( \boldsymbol{\pi}^{*} - \boldsymbol{\pi} \right)  \right)^2.
    \end{equation}
\end{proof}

\section{Proof that Result in Theorem \ref{LowerBoundTheoremImplicit} is No Weaker than that of Theorem \ref{LowerBoundTheorem_Explicit}} \label{LowerBoundTheorem_Implicit_Proof}
We can verify that the lower bound in Theorem \ref{LowerBoundTheoremImplicit} is no weaker than that in Theorem \ref{LowerBoundTheorem_Explicit}, by showing $\mathrm{Low}(\mathcal{A}) \geq \mathbf{H} $ below. 
Consider the following mathematical program, which is essentially the same as Program (\ref{Low(A)OptProblem}), except for replacing summation with maximization in the constraints.
\begin{equation} \label{Low(A)OptProblem_replace}
        \begin{aligned}
        &\mathrm{minimize}  &&\sum_{s = 1}^{d}\tau_s  \\ 
        &\mathrm{subject \ to} && \ \forall \boldsymbol{\pi} \in \mathcal{A}, \  \max_{s \in \boldsymbol{\pi}^{*} \diamond\boldsymbol{\pi} } \frac{\left| \pi^{*}_s - \pi_s \right|^2}{\tau_s}  \leq \left( \boldsymbol{\mu}^{\top}\left( \boldsymbol{\pi}^{*} - \boldsymbol{\pi} \right) \right)^2   \\
            & && \tau_s > 0, \forall s\in [d]. 
        \end{aligned}
\end{equation}
The optimal solution of (\ref{Low(A)OptProblem_replace}) is achieved by setting
\begin{equation}
    \tau'_s = \Delta_{(s)}^{-2}.
\end{equation}
Since every feasible solution of Program (\ref{Low(A)OptProblem}) is also a feasible solution of Program (\ref{Low(A)OptProblem_replace}), we have $\mathrm{Low}(\mathcal{A}) \geq \mathbf{H}$.
\section{Comparison with \citet{Nakamura2023}}
If $\mathrm{Gwidth}^{2}_1 \leq 2\mathrm{Gwidth}_2 $ and $\Delta_1 \approx \Delta_2 \approx \cdots \approx \Delta_d$, we have
\begin{eqnarray}
    \frac{\left(\sum_{e = 1}^{d} |\pi_e - \pi'_e|\right)^2}{32\Delta_s^2} \nonumber
    & \leq &  \frac{\mathrm{Gwidth}^2_1}{32\Delta_s^2} \\ \nonumber
    & \leq &  \frac{\mathrm{Gwidth}_2}{16\Delta_s^2} \\ \nonumber
    &=& \frac{|\pi_1 - \pi'_1|^2}{16\Delta_s^2} + \cdots + \frac{|\pi_d - \pi'_d|^2}{16\Delta_s^2} \nonumber \\
    & \approx & \frac{|\pi_1 - \pi'_1|^2}{16\Delta_1^2} + \cdots + \frac{|\pi_d - \pi'_d|^2}{16\Delta_d^2} \nonumber \\
    & \leq & \mathbf{H}. \nonumber
\end{eqnarray}
Therefore, our lower bound is tighter than that of \cite{Nakamura2023}.
\section{Proof of Proposition \ref{R-CPE-MAB_ArmSelectionStrategy_Explicit}}
Here, we prove Proposition \ref{R-CPE-MAB_ArmSelectionStrategy_Explicit}.
\begin{proof}
    For any $q\neq p^{\mathrm{R}}_t$, we have
    \begin{eqnarray}
        \sum_{s = 1}^{d}\frac{(\pi^{k}_s - \pi^{l}_s)^2}{T_k(t) + \boldsymbol{1}[s = q]} \geq \sum_{s = 1}^{d}\frac{(\pi^{k}_s - \pi^{l}_s)^2}{T_k(t) + \boldsymbol{1}[s = p^{\mathrm{R}}_t]}.
    \end{eqnarray}
    Thus, 
     \begin{equation}
            \frac{(\pi^{k}_q - \pi^{l}_q)^2}{T_q(t) + 1} + \frac{(\pi^{k}_{p^{\mathrm{R}}_t} - \pi^{l}_{p^{\mathrm{R}}_t})^2}{T_{p^{\mathrm{R}}_t}(t)} 
            \geq \frac{(\pi^{k}_q - \pi^{l}_q)^2}{T_q(t)} + \frac{(\pi^{k}_{p^{\mathrm{R}}_t} - \pi^{l}_{p^{\mathrm{R}}_t})^2}{T_{p^{\mathrm{R}}_t}(t) + 1}, \nonumber
    \end{equation}
    which implies
    \begin{equation}
            \frac{(\pi^{k}_{p^{\mathrm{R}}_t} - \pi^{l}_{p^{\mathrm{R}}_t})^2}{T_{p^{\mathrm{R}}_t}(t)(T_{p^{\mathrm{R}}_t}(t) + 1)} \geq  \frac{(\pi^{k}_{q} - \pi^{l}_{q})^2}{T_{q}(t)(T_{q}(t) + 1)}, \nonumber
    \end{equation}
    for all $q \neq p^{\mathrm{R}}_t$.
\end{proof}
\section{Proof of Theorem \ref{TSUpperBoundTheorem_R-CPE-MAB}} \label{TSUpperBoundTheorem_R-CPE-MABProof}
Here, we prove Theorem \ref{TSUpperBoundTheorem_R-CPE-MAB}. We define some notations. We denote $L_1(t) = 4R^2 \log\left( 12 |\mathcal{A}|^2 t^2/\delta \right)$ and $L_2(t) = \log\left( 12|\mathcal{A}|^2 t^2 M(\delta, q, t)/\delta \right)$ to simplify notations. We also denote $\mathcal{J} = \{ \boldsymbol{u} \in \mathbb{R}^d \ | \ \exists \boldsymbol{\pi}, \boldsymbol{\pi}' \in \mathcal{A}, \boldsymbol{u} = \boldsymbol{\pi} - \boldsymbol{\pi}' \}$, and $U_{\boldsymbol{\pi}, \boldsymbol{\pi}'} = \left\{ s\in[d] \ | \ \pi_s > \pi'_s \right\}$ for any $\boldsymbol{\pi}, \boldsymbol{\pi}'\in\mathcal{A}$. Note that $|\mathcal{J}| \leq |\mathcal{A}|^2$, and for any $\boldsymbol{u} \in \mathcal{J}$, $\sum_{s = 1}^d u_s \leq \mathrm{Gwidth}$. \par
We also define three events as follows:
$\mathcal{E}_{0}$ is the event that for all $t>0$, $\boldsymbol{u}\in \mathcal{J}$,
\begin{equation}
    \left| \sum_{s = 1}^{d} u_s (\hat{\mu}_s(t) - \mu_s) \right| \leq \sqrt{\sum_{s = 1}^{d} \frac{|u_s|^2}{2T_s(t)}L_1(t)};
\end{equation}
$\mathcal{E}_{1}$ is the event that for all $t>0$, $1\leq k\leq M(\delta, q, t)$, $\boldsymbol{u} \in \mathcal{J}$,
\begin{equation}
    \left| \sum_{s = 1}^{d} u_s (\theta^k_s(t) - \hat{\mu}_s) \right| \leq \sqrt{\sum_{s = 1}^{d} \frac{2|u_s|^2C(\delta, q, t)}{T_s(t)}L_2(t)};
\end{equation}
and $\mathcal{E}_2$ is the event that for all $t>0$, $\boldsymbol{\pi}, \boldsymbol{\pi}' \in \mathcal{J}$, there exists $1\leq k \leq M(\delta, q, t)$ such that
\begin{equation}
    \sum_{s = 1}^{d} (\pi_s - \pi'_s) \theta^k_s(t) \geq \sum_{s = 1}^{d} (\pi_s - \pi'_s) \hat{\mu}_s.
\end{equation}

\subsection{Useful Lemmas}
Here, we show some useful lemmas to prove Theorem \ref{TSUpperBoundTheorem_R-CPE-MAB}.
\begin{lemma}
    In Algorithm \ref{GenTS-ExploreAlgorithm}, we have that 
    \begin{equation}
        \Pr\left[ \mathcal{E}_{0,} \land \mathcal{E}_{1} \land \mathcal{E}_{2} \right] \geq 1 - \delta.
    \end{equation}
\end{lemma}
\begin{proof}
    Note that the random variable $(\hat{\mu}_s(t) - \mu_s)$ is zero-mean and $\sqrt{\frac{R}{T_s(t)}}$ -sub-Gaussian, and for different $s$, the random variables $(\hat{\mu}_s(t) - \mu_s)$'s are independent. Therefore, $\sum_{s = 1}^{d} u_s (\hat{\mu}_s(t) - \mu_s)$ is zero-mean and $\sqrt{\sum_{s = 1}^d \frac{|u_s|^2 }{T_s(t)}} R$ sub-Gaussian. Then, by concentration inequality of sub-Gaussian random variables, 
    \begin{eqnarray}
        &&\Pr\left[ \left| \sum_{s = 1}^{d} u_s (\hat{\mu}_s(t) - \mu_s) \right| > \sqrt{\sum_{s=1}^{d} \frac{|u_s|^2 }{2T_s(t)}L_1(t)}  \right] \nonumber \\
        &\leq& 2\exp\left( -L_1(t) \right) \nonumber \\
        &=& \frac{\delta}{6|\mathcal{A}|^2t^2}. \nonumber 
    \end{eqnarray}
    This implies that 
    \begin{equation}
        \Pr\left[ \lnot \mathcal{E}_{0} \right] \leq \sum_{\boldsymbol{u}, t} \frac{\delta}{6|\mathcal{A}|^2 t^2} \leq \sum_{t} \frac{\delta}{6t^2} \leq \frac{\delta}{3}
    \end{equation}
    Similarly, the random variables$(\theta_{s}^{k} - \hat{\mu}_{s}(t))$ is a zero-mean Gaussian random variable with variance $\frac{C(\delta, q, t)}{N_{s}(t)}$, and for different $s$, the random variables $(\theta_{s}^{k}(t) - \hat{\theta}_{s}(t))$'s are also independent. Then, by the concentration inequality, 
    \begin{eqnarray}
        &&
        \begin{split}
            \Pr\left[ \left| \sum_{s = 1}^{d} u_s \right. \right.& \left. (\theta_{s}^{k}(t) -  \hat{\theta}_{s}(t))  \right|  \\ 
            & \left. > \sqrt{\sum_{s = 1}^{d} \frac{2|u_s|^2 C(\delta, q, t)}{ T_s(t)}L_2(t)}  \right]
        \end{split} \nonumber \\
        & \leq & 2\exp(- L_2(t)) \nonumber \\
        & = & \frac{\delta}{6|\mathcal{A}|^2t^2M(\delta, q, t)}. \nonumber 
    \end{eqnarray}
    This implies that
    \begin{eqnarray}
        \Pr\left[ \lnot \mathcal{E}_1 \right] \leq \sum_{\boldsymbol{u}, t, k} \frac{\delta}{6|\mathcal{A}|^2 t^2 M(\delta, q, t)} \leq \sum_{\boldsymbol{u}, t} \frac{\delta}{6|\mathcal{A}|^2t^2} \leq \frac{\delta}{3}, \nonumber 
    \end{eqnarray}
    where the second inequality is because that there are totally $M(\delta, q, t)$ sample sets at time step $t$. \par
    Finally, we consider the probability $\Pr\left[ \lnot \mathcal{E}_2 | \mathcal{E}_0 \right]$. In the following, we denote 
    \begin{eqnarray}
    \Delta_{\boldsymbol{\pi}, \boldsymbol{\pi}'} &\triangleq& \sum_{s = 1}^{d} \pi_s \mu_s - \sum_{s = 1}^{d} \pi'_s\mu_s \nonumber \\
    &=& \sum_{s = 1}^{d} (\pi_s - \pi'_s) \mu_s - \sum_{s = 1}^{d} (\pi'_s - \pi_s) \mu_s    
    \end{eqnarray}
    as the reward gap. Additionally, we denote $A(t) =  \sum_{s = 1}^{d} \frac{(\pi_s - \pi'_s)^2}{2T_s(t)}$, $B(t) = \sum_{s = 1}^{d} \frac{(\pi_s - \pi'_s)^2}{2T_s(t)}$, and $C(t) = C(\delta, q, t)$ to simplify notations. Then, under event $\mathcal{E}_0$, we have the following:
    \begin{eqnarray}
        &&\sum_{s = 1}^{d} (\pi_s - \pi'_s) \hat{\mu}_s(t) - \sum_{s = 1}^{d} (\pi'_s - \pi_s) \hat{\mu}_s  \\
        &\geq&
        \begin{split}
            \biggl(\sum_{s = 1}^{d} (\pi_s &- \pi'_s)\mu_s - \sqrt{A(t)L_1(t)} \biggr) \\
            &- \left(\sum_{s = 1}^{d} (\pi'_s - \pi_s)\mu_s + \sqrt{B(t)L_1(t)} \right) 
        \end{split} \nonumber \\
        & \geq & \Delta_{\boldsymbol{\pi}, \boldsymbol{\pi}'} - \sqrt{A(t)L_1(t)} - \sqrt{B(t)L_1(t)}.
    \end{eqnarray}
    Since $\sum_{s = 1}^{d} (\pi_s - \pi'_s) \theta^k_s(t) - \sum_{s = 1}^{d} (\pi_s - \pi'_s) \theta_s^k(t)$ is a Gaussian random variable with mean $\sum_{s = 1}^{d} (\pi_s - \pi'_s) \hat{\mu}_s(t) - \sum_{s = 1}^{d} (\pi_s - \pi'_s) \hat{\mu}_s(t)$ and variance $2A(t)C(t) + 2B(t)C(t)$, then under event $\mathcal{E}_0$,
    \begin{eqnarray}
        &&\Pr\left[ \sum_{s = 1}^d \pi_s \theta_s^k(t) - \sum_{s = 1}^{d} \pi'_s \theta_s^{k}(t) \geq \Delta_{\boldsymbol{\pi}, \boldsymbol{\pi}'}  \right] \nonumber \\
        &=&\Pr\left[ \sum_{s = 1}^{d} (\pi_s - \pi'_s) \theta_{s}^k(t) - \sum_{s = 1}^{d} (\pi'_s  - \pi_s) \theta_s^k(t) \geq \Delta_{\boldsymbol{\pi}, \boldsymbol{\pi}'} \right] \nonumber \\
        &=& 
        \begin{split}
            \Phi\biggl( \Delta_{\boldsymbol{\pi}, \boldsymbol{\pi}'}, \sum_{s = 1}^{d} (\pi_s - \pi'_s) & \hat{\mu}_s(t) - \sum_{s = 1}^{d} (\pi'_s - \pi_s) \hat{\mu}(t), \\
            &2A(t)C(t) + 2B(t)C(t) \biggr) 
        \end{split} \nonumber \\
        & \geq &
        \begin{split}
            \Phi\biggl( \Delta_{\boldsymbol{\pi}, \boldsymbol{\pi}'}, \Delta_{\boldsymbol{\pi}, \boldsymbol{\pi}'} - \sqrt{A(t)L_1(t)} - \sqrt{B(t)L_1(t)}, \\
            2A(t)C(t) + 2B(t)C(t) \biggr)
        \end{split} \nonumber \\
        & = &
        \begin{split}
            \Phi\biggl( \sqrt{A(t)L_1(t)} &+ \sqrt{B(t)L_1(t)} , 0, \\
            &2A(t)C(t) + 2B(t)C(t) \biggr)
        \end{split} \nonumber \\
        & = & \Phi\left( \sqrt{\frac{L_1(t)}{C(t)}} \cdot \frac{\sqrt{A(t)} + \sqrt{B(t)}}{\sqrt{2A(t) + 2B(t)}} , 0, 1 \right) \\
        & \geq & \Phi\left( \sqrt{\frac{L_1(t)}{C(t)}}, 0, 1 \right) \\
        & = & q,
    \end{eqnarray}
    where the last equation holds because that we choose $C(t) = \frac{L_1(t)}{\phi^2(q)}$ and $\Phi\left( \phi(q), 0, 1 \right) = q$ (by the definition of $\phi$). \par
    Note that the parameter sets $\left\{ \boldsymbol{\theta}^{k}(t) \right\}_{k = 1}^{M(\delta, q, t)}$ are chosen independently, therefore under event $\mathcal{E}_{0}$, we have that
    \begin{eqnarray}
        &&\Pr\left[ \forall k, \sum_{s = 1}^{d} \pi_s \theta_{s}^{k}(t) - \sum_{s = 1}^{d} \pi' \theta_{s}^{k}(t) < \Delta_{\boldsymbol{\pi}, \boldsymbol{\pi}'}\right] \\
        &\leq& (1 - q)^{M(\delta, q, t)} \\ 
        &\leq& \frac{\delta}{12|\mathcal{A}|^2t^2},
    \end{eqnarray}
    where the last inequality is because that we choose $M(\delta, q, t) = \frac{1}{q} \log\left( \frac{12|\mathcal{A}|^2 t^2}{\delta}  \right)$. \par
    This implies that 
    \begin{equation}
        \Pr\left[ \lnot \mathcal{E}_2 \ | \ \mathcal{E}_0 \right] \leq \sum_{t, \boldsymbol{\pi}, \boldsymbol{\pi}'} \frac{\delta}{12|\mathcal{A}|^2 t^2} \leq \sum_{t} \frac{\delta}{12t^2} \leq \frac{\delta}{3}.
    \end{equation}
    The above discussion show that $\Pr\left[ \mathcal{E} \land \mathcal{E}_1 \land \mathcal{E}_2 \right] \geq 1 - \delta$.
\end{proof}
\subsection{An Upper Bound of the R-CPE-MAB Arm Selection Strategy}
Recall that we pull arm $p_t = \argmax_{s\in\{e\in [d] \ | \ \hat{\pi}_e \neq \Tilde{\pi}_e\}} \frac{|\hat{\pi}_e - \Tilde{\pi}_e|^2}{T_s(t) (T_s(t) + 1)}$ at round $t$. For any $s \in \{e\in [d] \ | \ \hat{\pi}_e \neq \Tilde{\pi}_e\}$, we have
\begin{equation}
    \frac{|\hat{\pi}_{p_t} - \Tilde{\pi}_{p_t}|^2}{T_{p_t}^2} > \frac{|\hat{\pi}_{p_t} - \Tilde{\pi}_{p_t}|^2}{T_{p_t} (T_{p_t} + 1)} \geq \frac{|\hat{\pi}_s - \Tilde{\pi}_s|^2}{T_{s} (T_{s} + 1)} > \frac{|\hat{\pi}_s - \Tilde{\pi}_s|^2}{(T_{s} + 1)^2},
\end{equation}
and therefore, we have 
\begin{eqnarray}
    T_s(t) + 1 >  \frac{|\hat{\pi}_s - \Tilde{\pi}_s|}{|\hat{\pi}_{p_t} - \Tilde{\pi}_{p_t}|} T_{p_t}.
\end{eqnarray}
We have the following lemma.
\begin{lemma} \label{Lemma4.7GenTS-Explore}
    Under event $\mathcal{E}_0 \land \mathcal{E}_1 \land \mathcal{E}_2$, an arm $p$ will not be pulled if $T_p(t) \geq \frac{98V_p C(t) L_2(t)}{\Delta^2_{(p)}} + \frac{1}{A_p}$, where $A_p = \max_{\boldsymbol{\pi}, \boldsymbol{\pi}' \in \mathcal{A}, v \in [d], \boldsymbol{\pi}_p \neq \boldsymbol{\pi}'_p} \frac{|{\pi}_{v} - {\pi}'_{v}|}{|{\pi}_p - {\pi}'_p|}$.
\end{lemma}
\begin{proof}
    Here, we simply write $\hat{\boldsymbol{\pi}}$ instead of $\hat{\boldsymbol{\pi}}(t)$ and $\Tilde{\boldsymbol{\pi}}(t)$ instead of $\Tilde{\boldsymbol{\pi}}$.
    We prove this lemma by contradiction. Assume that arm $p$ is pulled with $T_p(t) \geq $. Then, there are two cases.
    \begin{itemize}
        \item Case 1: $\Tilde{{\pi}}_p \neq {\pi}^{*}_p$
        \item Case 2: $\hat{{\pi}}_p \neq {\pi}^{*}_p$
    \end{itemize}
    Case 1: $\Tilde{{\pi}}_p \neq {\pi}^{*}_p$. In this case, we have $\Delta_{(p)} \leq \frac{\boldsymbol{\mu}^{\top} \left( \boldsymbol{\pi}^{*} - \Tilde{\boldsymbol{\pi}} \right)}{\left|{\pi}_p^{*} - \Tilde{\pi}_p \right| }$.
    By event $\mathcal{E}_1$, we also have that, for any $k$, 
    \begin{eqnarray}
        && \frac{1}{\left| {\pi}^{*}_p - \Tilde{\pi}_p \right|} \left| \sum_{s \in U_{\Tilde{\boldsymbol{\pi}}, \hat{\boldsymbol{\pi}}}} (\Tilde{\pi}_s - \hat{\pi}_s)(\theta^{k}_{s}(t) - \hat{\mu}_{s}(t)) \right| \nonumber \\
        &\leq& \frac{1}{\left| {\pi}^{*}_p - \Tilde{\pi}_p \right|} \sqrt{C(t)L_2(t) \sum_{s \in U_{\Tilde{\boldsymbol{\pi}}, \hat{\boldsymbol{\pi}}}}  \frac{2|\Tilde{\pi}_s - \hat{\pi}_s|^2}{T_s(t)}} \nonumber \\
        & \leq  & \sqrt{C(t)L_2(t) \sum_{s \in U_{\Tilde{\boldsymbol{\pi}}, \hat{\boldsymbol{\pi}}}}  \frac{2|\Tilde{\pi}_s - \hat{\pi}_s|^2}{ \left| {\pi}^{*}_p - \Tilde{\pi}_p \right|^2 T_s(t)}} \nonumber \\
        & \leq & \frac{\Delta_{(p)}}{7} \label{third_inequality_1},
    \end{eqnarray}
    and similarly, 
    \begin{eqnarray}
        && \frac{1}{\left| {\pi}^{*}_p - \Tilde{\pi}_p \right|} \left| \sum_{s \in U_{\hat{\boldsymbol{\pi}}, \Tilde{\boldsymbol{\pi}}}} (\Tilde{\pi}_s - \hat{\pi}_s)(\theta^{k}_{s}(t) - \hat{\mu}_{s}(t)) \right| \nonumber \\
        &\leq& \sqrt{C(t)L_2(t) \sum_{s \in U_{\Tilde{\boldsymbol{\pi}}, \hat{\boldsymbol{\pi}}}}  \frac{2|\pi_s - \hat{\pi}_s|^2}{\left| {\pi}^{*}_p - \Tilde{\pi}_p \right|^2 T_s(t)}} \nonumber \\
        & \leq & \frac{\Delta_{(p)}}{7} \label{third_inequality_2}.
    \end{eqnarray}
    Therefore, for any $k$, we have that
    \begin{eqnarray}
        &&\frac{1}{\left| {\pi}^{*}_p - \Tilde{\pi}_p \right|} \sum_{s = 1}^{d} (\Tilde{\pi}_s - \hat{\pi}_s) \theta^{k}_s(t) \nonumber \\
        & = &
        \begin{split}
            \frac{1}{\left| {\pi}^{*}_p - \Tilde{\pi}_p \right|} \biggl( 
            &\sum_{s = 1}^{d} (\Tilde{\pi}_s - \hat{\pi}_s)
            \hat{\mu}_s(t) \\
            &+ \sum_{s \in U_{\Tilde{\boldsymbol{\pi}}, \hat{\boldsymbol{\pi}}}} (\Tilde{\pi}_s - \hat{\pi}_s) (\theta^{k}_s(t) - \hat{\mu}_s(t)) \\
            &+ \sum_{s \in U_{\hat{\boldsymbol{\pi}}, \Tilde{\boldsymbol{\pi}}}} (\hat{\pi}_s - \Tilde{\pi}_s) (\theta^{k}_s(t) - \hat{\mu}_s(t)) \biggr)
        \end{split}  \nonumber \\
        & \leq &
        \begin{split}
            \frac{1}{\left| {\pi}^{*}_p - \Tilde{\pi}_p \right|} \biggl( & \sum_{s = 1}^{d} (\Tilde{\pi}_s - \hat{\pi}_s)
            \hat{\mu}_s(t) \\
            &+ \biggl| \sum_{s \in U_{\Tilde{\boldsymbol{\pi}}, \hat{\boldsymbol{\pi}}}} (\Tilde{\pi}_s - \hat{\pi}_s) (\theta^{k}_s(t) - \hat{\mu}_s(t)) \biggr| \\
            &+ \biggl| \sum_{s \in U_{\hat{\boldsymbol{\pi}}, \Tilde{\boldsymbol{\pi}}}} (\hat{\pi}_s - \Tilde{\pi}_s) (\theta^{k}_s(t) - \hat{\mu}_s(t)) \biggr| \biggr)
        \end{split} \nonumber \\
        & \leq & 0 + \frac{2\Delta_{(p)}}{7} = \frac{2\Delta_{(p)}}{7}. \nonumber
     \end{eqnarray}
     This means that $\frac{\Tilde{\Delta}^{k_t^{*}}_t}{\left| \Tilde{\pi}_p - \hat{\pi}_p \right|} \leq \frac{2\Delta_{(p)}}{7}$. \par
     Moreover, since $\sum_{s = 1}^{d} \Tilde{\pi}_s \theta^{k^{*}_t}_s \geq \sum_{s = 1}^{d} \hat{\pi}_s \theta^{k^{*}_t}_s$, we have that
     \begin{eqnarray}
     && \frac{1}{\left| {\pi}^{*}_p - \Tilde{\pi}_p \right|} \sum_{s = 1}^{d} (\Tilde{\pi}_s - \hat{\pi}_s) \hat{\mu}_{s}(t) \nonumber \\
        & = &
        \begin{split}
            \frac{1}{\left| {\pi}^{*}_p - \Tilde{\pi}_p \right|} \biggl( & \sum_{s = 1}^{d} (\Tilde{\pi}_s - \hat{\pi}_s)
            \theta^{k^{*}_{t}}_s(t) \\ 
            &- \sum_{s \in U_{\Tilde{\boldsymbol{\pi}}, \hat{\boldsymbol{\pi}}}} (\Tilde{\pi}_s - \hat{\pi}_s) (\theta^{k^{*}_{t}}_s(t) - \hat{\mu}_s(t)) \\
            &+ \sum_{s \in U_{\hat{\boldsymbol{\pi}}, \Tilde{\boldsymbol{\pi}}}} (\hat{\pi}_s - \Tilde{\pi}_s) (\theta^{k^{*}_{t}}_s(t) - \hat{\mu}_s(t)) \biggl)
        \end{split} \nonumber \\
        & \geq &
        \begin{split}
        \frac{1}{\left| {\pi}^{*}_p - \Tilde{\pi}_p \right|} \biggl( & \sum_{s = 1}^{d} (\Tilde{\pi}_s - \hat{\pi}_s)
        \theta^{k^{*}_{t}}_s(t) \\
        &- \biggl| \sum_{s \in U_{\Tilde{\boldsymbol{\pi}}, \hat{\boldsymbol{\pi}}}} (\Tilde{\pi}_s - \hat{\pi}_s) (\theta^{k^{*}_{t}}_s(t)  - \hat{\mu}_s(t)) \biggr| \\
            &- \biggl| \sum_{s \in U_{\hat{\boldsymbol{\pi}}, \Tilde{\boldsymbol{\pi}}}} (\hat{\pi}_s - \Tilde{\pi}_s) (\theta^{k^{*}_{t}}_s(t) - \hat{\mu}_s(t)) \biggr| \biggr)
        \end{split} \nonumber \\
        & \geq & 0 - \frac{2\Delta_{(p)}}{7} =  -\frac{2\Delta_{(p)}}{7}. \nonumber
     \end{eqnarray}
     On the other hand, by event $\mathcal{E}_{2}$, we know that there exists a $k'$ such that $\sum_{s = 1}^{d} \pi^{*}_s \theta^{k'}_s(t) - \sum_{s = 1}^{d} \Tilde{\pi}_s \theta^{k'}_s(t) \geq \Delta_{\boldsymbol{\pi}^{*}, \Tilde{\boldsymbol{\pi}}}$. 
     Then,
     \begin{eqnarray}
        &&\frac{1}{\left| {\pi}^{*}_p - \Tilde{\pi}_p \right|} \left(\sum_{s = 1}^{d} \pi^{*}_s \theta^{k'}_s(t) - \sum_{s = 1}^{d} \hat{\pi}_s \theta^{k'}_s(t) \right) \nonumber \\
        & = &
        \begin{split}
           \frac{1}{\left| {\pi}^{*}_p - \Tilde{\pi}_p \right|} \Biggl( \biggl( & \sum_{s = 1}^{d} \pi^{*}_s \theta^{k'}_s(t) - \sum_{s = 1}^{d} \Tilde{\pi}_s \theta^{k'}_s(t) \biggr)  \\
            & + \biggl(\sum_{s = 1}^{d} \Tilde{\pi}_s \theta^{k'}_s(t) - \sum_{s = 1}^{d} \hat{\pi}_s \theta^{k'}_s(t) \biggr) \Biggr)
        \end{split} \nonumber \\
        & \geq &
            \frac{\Delta_{\boldsymbol{\pi}^{*}, \Tilde{\boldsymbol{\pi}}}}{\left| {\pi}^{*}_p - \Tilde{\pi}_p \right|}  + \frac{1}{\left| {\pi}^{*}_p - \Tilde{\pi}_p \right|} \biggl( \sum_{s = 1}^{d} \Tilde{\pi}_s \theta^{k'}_s(t) - \sum_{s = 1}^{d} \hat{\pi}_s \theta^{k'}_s(t) \biggr) \nonumber \\
        & \geq &
        \begin{split}
            \Delta_{(p)} + \frac{1}{\left| {\pi}^{*}_p - \Tilde{\pi}_p \right|} \biggl( & \sum_{s \in U_{\Tilde{\boldsymbol{\pi}}, \hat{\boldsymbol{\pi}}}} \left(\Tilde{\pi}_s(t) -  \hat{\pi}_s)(t) \right)  \theta^{k'}_s(t)  \\ 
            &- \sum_{s \in U_{\hat{\boldsymbol{\pi}}, \Tilde{\boldsymbol{\pi}}}} \left(\hat{\pi}_s (t) - \Tilde{\pi}_s(t)\right) \theta^{k'}_s(t) \biggr)
        \end{split} \nonumber \\
        & = &
         \begin{split} 
               & \Delta_{(p)}  + \frac{1}{\left| {\pi}^{*}_p - \Tilde{\pi}_p \right|} \biggl( \sum_{s \in U_{\Tilde{\boldsymbol{\pi}}, \hat{\boldsymbol{\pi}}}} \left(\Tilde{\pi}_s(t) -  \hat{\pi}_s)(t) \right)  \hat{\mu}_s(t) \\
               & -  \sum_{s \in U_{\hat{\boldsymbol{\pi}}, \Tilde{\boldsymbol{\pi}}}} \left(\hat{\pi}_s (t) - \Tilde{\pi}_s(t)\right) \hat{\mu}_s(t) \biggr)  \\
               & + \frac{1}{\left| {\pi}^{*}_p - \Tilde{\pi}_p \right|} \left(\sum_{s \in U_{\Tilde{\boldsymbol{\pi}}, \hat{\boldsymbol{\pi}}}} \left(\Tilde{\pi}_s(t) -  \hat{\pi}_s)(t) \right)  (\theta^{k'}_{s}(t) - \hat{\mu}_s(t)) \right) \\
               & - \frac{1}{\left| {\pi}^{*}_p - \Tilde{\pi}_p \right|} \biggl( \sum_{s \in U_{\hat{\boldsymbol{\pi}}, \Tilde{\boldsymbol{\pi}}}} \left(\hat{\pi}_s (t) - \Tilde{\pi}_s(t)\right) (\theta^{k'}_{s}(t) - \hat{\mu}_s(t)) \biggr)
          \end{split} \nonumber \\
        & \geq  &
        \begin{split} 
               &\Delta_{(p)} - \frac{2\Delta_{(p)}}{7}  \\
               &- \frac{1}{\left| {\pi}^{*}_p - \Tilde{\pi}_p \right|} \left|\sum_{s \in U_{\Tilde{\boldsymbol{\pi}}, \hat{\boldsymbol{\pi}}}} \left(\Tilde{\pi}_s(t) -  \hat{\pi}_s(t) \right)  (\theta^{k'}_{s}(t) - \hat{\mu}_s(t)) \right| \\
               &- \frac{1}{\left| {\pi}^{*}_p - \Tilde{\pi}_p \right|} \left| \sum_{s \in U_{\hat{\boldsymbol{\pi}}, \Tilde{\boldsymbol{\pi}}}} \left(\hat{\pi}_s (t) - \Tilde{\pi}_s(t) \right) (\theta^{k'}_{s}(t) - \hat{\mu}_s(t)) \right|
        \end{split} \nonumber \\
        & \geq & \Delta_{(p)} - \frac{4\Delta_{(p)}}{7} \nonumber \\
        & > & \frac{2\Delta_{(p)}}{7}. \nonumber
     \end{eqnarray}
     This means that $\frac{\Tilde{\Delta}^{k_t^{*}}_t}{\left| \Tilde{\pi}_p - \hat{\pi}_p \right|} > \frac{2\Delta_{(p)}}{7}$, which contradicts with $\frac{\Tilde{\Delta}^{k_t^{*}}_t}{\left| \Tilde{\pi}_p - \hat{\pi}_p \right|} \leq \frac{2\Delta_{(p)}}{7}$. \par
     
     Case 2: $\hat{{\pi}}_p \neq {\pi}^{*}_p$. 
     In this case, $\Delta_{(p)} \leq \frac{\boldsymbol{\mu}^{\top} \left( \boldsymbol{\pi}^{*} - \hat{\boldsymbol{\pi}} \right)}{\left|{\pi}_p^{*} - \hat{\pi}_p(t) \right| }$. 
     By event $\mathcal{E}_2$, we know that there exists a $k$ such that $\sum_{s = 1}^{d} \pi^{*}_s \theta_s^{k}(t) - \sum_{s = 1}^{d} \hat{\pi}_s(t) \theta^{k}_s(t) \geq \Delta_{\boldsymbol{\pi}^{*}, \hat{\boldsymbol{\pi}}}$. 
     Hence, $\Tilde{\Delta}_t^k \geq \Delta_{\boldsymbol{\pi}^{*}, \hat{\boldsymbol{\pi}}}$.  
     Moreover, since $k^{*}_t = \argmax_k \Tilde{\Delta}_t^k$, we have that $ \Tilde{\Delta}^{k^{*}_t}_t = \sum_{s = 1}^{d} \pi^{*}_s \theta_s^{k^{*}_t}(t) - \sum_{s = 1}^{d} \hat{\pi}_s(t) \theta^{k^{*}_t}_s(t) \geq \Delta_{\boldsymbol{\pi}^{*}, \hat{\boldsymbol{\pi}}}$, 
     which is the same as $\sum_{s \in U(\boldsymbol{\pi}^{*}, \hat{\boldsymbol{\pi}})} (\pi^{*}_s - \hat{\pi}_s) \theta_s^{k^{*}_t}(t) - \sum_{s = 1}^{d} (\hat{\pi}_s(t) - \pi^{*}) \theta^{k^{*}_t}_s(t) \geq \Delta_{\boldsymbol{\pi}^{*}, \hat{\boldsymbol{\pi}}}$. 
     On the other hand, by event $\mathcal{E}_1$, we also have that, for any $k$, 
    \begin{eqnarray}
        && \frac{1}{\left| {\pi}^{*}_p - \hat{\pi}_p \right|} \left| \sum_{s \in U_{\Tilde{\boldsymbol{\pi}}, \hat{\boldsymbol{\pi}}}} (\Tilde{\pi}_s - \hat{\pi}_s)(\theta^{k^{*}_t}_{s}(t) - \hat{\mu}_{s}(t)) \right| \nonumber \\
        &\leq& \frac{1}{\left| {\pi}^{*}_p - \hat{\pi}_p \right|} \sqrt{C(t)L_2(t) \sum_{s \in U_{\Tilde{\boldsymbol{\pi}}, \hat{\boldsymbol{\pi}}}}  \frac{2|\Tilde{\pi}_s - \hat{\pi}_s|^2}{T_s(t)}} \nonumber \\
        & \leq & \sqrt{C(t)L_2(t) \sum_{s \in U_{\Tilde{\boldsymbol{\pi}}, \hat{\boldsymbol{\pi}}}}  \frac{2|\Tilde{\pi}_s - \hat{\pi}_s|^2}{\left| {\pi}^{*}_p - \hat{\pi}_p \right|^2 T_s(t)}} \nonumber \\
        & \leq & \frac{\Delta_{(p)}}{7}, \nonumber 
    \end{eqnarray}
    and similarly, 
    \begin{eqnarray}
        && \frac{1}{\left| {\pi}^{*}_p - \hat{\pi}_p \right|} \left| \sum_{s \in U_{\hat{\boldsymbol{\pi}}, \Tilde{\boldsymbol{\pi}}}} (\Tilde{\pi}_s - \hat{\pi}_s)(\theta^{k^{*}}_{s}(t) - \hat{\mu}_{s}(t)) \right| \nonumber \\
        &\leq& \frac{1}{\left| {\pi}^{*}_p - \hat{\pi}_p \right|} \sqrt{C(t)L_2(t) \sum_{s \in U_{\Tilde{\boldsymbol{\pi}}, \hat{\boldsymbol{\pi}}}}  \frac{2|\Tilde{\pi}_s - \hat{\pi}_s|^2}{T_s(t)}} \nonumber \\
        & \leq & \sqrt{C(t)L_2(t) \sum_{s \in U_{\hat{\boldsymbol{\pi}}, \Tilde{\boldsymbol{\pi}}}}  \frac{2|\Tilde{\pi}_s - \hat{\pi}_s|^2}{\left| {\pi}^{*}_p - \hat{\pi}_p \right|^2 T_{p}}} \nonumber \\
        & \leq & \frac{\Delta_{{(p)}}}{7}. \nonumber 
    \end{eqnarray}
    Therefore, 
    \begin{eqnarray}
       && \frac{1}{\left| {\pi}^{*}_p - \hat{\pi}_p \right|} \left( \sum_{s \in U_{\Tilde{\pi}, \hat{\boldsymbol{\pi}}}} (\Tilde{\pi}_s - \hat{\pi}_s) \hat{\mu}_s - \sum_{s \in U_{\hat{\boldsymbol{\pi}}, \Tilde{\boldsymbol{\pi}}}} \left(\hat{\pi}_s - \Tilde{\pi}_s \right) \hat{\mu}_s \right) \nonumber \\
       & = &
         \begin{split} 
                &\frac{1}{\left| {\pi}^{*}_p - \hat{\pi}_p \right|} \times \\
                &\Biggl( \biggl(\sum_{s \in U_{\Tilde{\boldsymbol{\pi}}, \hat{\boldsymbol{\pi}}}} \left(\Tilde{\pi}_s(t) -  \hat{\pi}_s(t) \right)  \theta^{k^{*}_t}_s(t) \\ 
               &- \sum_{s \in U_{\hat{\boldsymbol{\pi}}, \Tilde{\boldsymbol{\pi}}}} \left(\hat{\pi}_s (t) - \Tilde{\pi}_s(t)\right) \theta^{k^{*}_t}_s(t) \biggr)  \\
               &- \biggl(\sum_{s \in U_{\Tilde{\boldsymbol{\pi}}, \hat{\boldsymbol{\pi}}}} \left(\Tilde{\pi}_s(t) -  \hat{\pi}_s)(t) \right)  (\theta^{k^{*}_t}_{s}(t) - \hat{\mu}_s(t)) \biggr) \\
               &+ \biggl( \sum_{s \in U_{\hat{\boldsymbol{\pi}}, \Tilde{\boldsymbol{\pi}}}} \left(\hat{\pi}_s (t) - \Tilde{\pi}_s(t)\right) (\theta^{k^{*}_t}_{s}(t) - \hat{\mu}_s(t)) \biggr) \Biggl)
          \end{split} \nonumber \\
        & = &
         \begin{split} 
               & \frac{1}{\left| {\pi}^{*}_p - \hat{\pi}_p \right|} \biggl(\sum_{s \in U_{\Tilde{\boldsymbol{\pi}}, \hat{\boldsymbol{\pi}}}} \left(\Tilde{\pi}_s(t) -  \hat{\pi}_s(t) \right)  \theta^{k^{*}_t}_s(t) \\
               & -  \sum_{s \in U_{\hat{\boldsymbol{\pi}}, \Tilde{\boldsymbol{\pi}}}} \left(\hat{\pi}_s (t) - \Tilde{\pi}_s(t)\right) \theta^{k^{*}_t}_s(t) \biggr)  \\
               & - \frac{1}{\left| {\pi}^{*}_p - \hat{\pi}_p \right|} \left|\sum_{s \in U_{\Tilde{\boldsymbol{\pi}}, \hat{\boldsymbol{\pi}}}} \left(\Tilde{\pi}_s(t) -  \hat{\pi}_s)(t) \right)  (\theta^{k^{*}_t}_{s}(t) - \hat{\mu}_s(t)) \right| \\
               & - \frac{1}{\left| {\pi}^{*}_p - \hat{\pi}_p \right|} \left| \sum_{s \in U_{\hat{\boldsymbol{\pi}}, \Tilde{\boldsymbol{\pi}}}} \left(\hat{\pi}_s (t) - \Tilde{\pi}_s(t)\right) (\theta^{k^{*}_t}_{s}(t) - \hat{\mu}_s(t)) \right|
          \end{split} \nonumber \\
        & \geq & \frac{\Delta_{\boldsymbol{\pi}^{*}, \hat{\boldsymbol{\pi}}}}{\left| {\pi}^{*}_p - \hat{\pi}_p \right|} - \frac{2\Delta_{(p)}}{7} \nonumber \\
        & \geq & \Delta_{(p)} - \frac{2\Delta_{(p)}}{7} \nonumber \\
        & > & 0 \nonumber 
    \end{eqnarray} 
    This means that $\sum_{s 1}^{d} \Tilde{\pi}_s \hat{\mu}_s - \sum_{s = 1}^{d} \hat{\pi}_s \hat{\mu}_s  > 0$, which contradicts with the fact that $\hat{\boldsymbol{\pi}} = \argmax_{\boldsymbol{\pi} \in \mathcal{A}} \boldsymbol{\mu}^{\top}\boldsymbol{\pi}$.
\end{proof}

\subsection{Obtain Theorem \ref{TSUpperBoundTheorem_R-CPE-MAB}}
By Lemma \ref{Lemma4.7GenTS-Explore}, if for all $s\in[d]$, $T_s(t) \geq \frac{98\mathrm{Gwidth}_s C(t) L_2(t)}{\Delta^2_{(s)}}  + \frac{1}{A_s}$, then the GenTS-Explore algorithm must terminate. Thus, the complexity $Z$ satisfies 
\begin{eqnarray}
    Z 
    &\leq& \sum_{s = 1}^{d} \frac{98\mathrm{Gwidth}_s C(t) L_2(t)}{\Delta^2_{(s)}} + D \nonumber \\ 
    &=& 98\mathbf{H}^{R} C(Z)L_2(Z) + D,
\end{eqnarray}
where $D = \sum_{s = 1}^{d} \frac{1}{A_s}$. For, $q = 0.1$, $\frac{1}{\phi^2(q)} = \mathcal{O}\left( \frac{1}{\log\frac{1}{q}} \right)$. Then, with $C(Z) = \frac{\log\left( 12\left| \mathcal{A} \right| Z^2 / \delta \right)}{\phi^2(q)}$, $L_2(Z) = \log\left( 12\left| \mathcal{A} \right|^2 Z^2 M(Z) / \delta \right)$, and $M(Z) = \frac{\log\left( 12\left| \mathcal{A} \right| Z^2 / \delta \right)}{q}$, we have 
\begin{eqnarray}
    Z \leq \mathcal{O}\left( \mathbf{H}^{\mathrm{R}} \frac{\left( \log\left( \left| \mathcal{A} \right|Z \right) + \log \frac{1}{\delta} \right)^2 }{\log\frac{1}{q}} \right).
\end{eqnarray}
Following the same discussion in Section B in \citet{WangAndZhu2022}, we obtain the Theorem.

\subsection{An Upper Bound of the Naive Arm Selection Strategy}
Recall that we pull arm $p_t = \argmin_{s\in\{e\in [d] \ | \ \hat{\pi}_e \neq \Tilde{\pi}_e\}} T_s(t)$ at round $t$. Therefore, we have 
\begin{eqnarray}
    T_{p_t}(t) \leq T_s(t)
\end{eqnarray}
for any $s\in [d] \setminus \{p_t\}$. \par
We have the following lemma.
\begin{lemma} \label{Lemma4.7GenTS-Explore}
    Under event $\mathcal{E}_0 \land \mathcal{E}_1 \land \mathcal{E}_2$, an arm $p$ will not be pulled if $T_p(t) \geq \frac{98U_p C(t) L_2(t)}{\Delta^2_{(p)}}$.
\end{lemma}
\begin{proof}
    Here, we simply write $\hat{\boldsymbol{\pi}}$ instead of $\hat{\boldsymbol{\pi}}(t)$ and $\Tilde{\boldsymbol{\pi}}(t)$ instead of $\Tilde{\boldsymbol{\pi}}$.
    We prove this lemma by contradiction. Assume that arm $p$ is pulled with $T_p(t) \geq \frac{98U_p C(t) L_2(t)}{\Delta^2_{(p)}}$. Then, there are two cases.
    \begin{itemize}
        \item Case 1: $\Tilde{{\pi}}_p \neq {\pi}^{*}_p$
        \item Case 2: $\hat{{\pi}}_p \neq {\pi}^{*}_p$
    \end{itemize}
    Case 1: $\Tilde{{\pi}}_p \neq {\pi}^{*}_p$. In this case, we have $\Delta_{(p)} \leq \frac{\boldsymbol{\mu}^{\top} \left( \boldsymbol{\pi}^{*} - \Tilde{\boldsymbol{\pi}} \right)}{\left|{\pi}_p^{*} - \Tilde{\pi}_p \right| }$.
    By event $\mathcal{E}_1$, we also have that, for any $k$, 
    \begin{eqnarray}
        && \frac{1}{\left| {\pi}^{*}_p - \Tilde{\pi}_p \right|} \left| \sum_{s \in U_{\Tilde{\boldsymbol{\pi}}, \hat{\boldsymbol{\pi}}}} (\Tilde{\pi}_s - \hat{\pi}_s)(\theta^{k}_{s}(t) - \hat{\mu}_{s}(t)) \right| \nonumber \\
        &\leq& \frac{1}{\left| {\pi}^{*}_p - \Tilde{\pi}_p \right|} \sqrt{C(t)L_2(t) \sum_{s \in U_{\Tilde{\boldsymbol{\pi}}, \hat{\boldsymbol{\pi}}}}  \frac{2|\Tilde{\pi}_s - \hat{\pi}_s|^2}{T_s(t)}} \nonumber \\
        & \leq  & \sqrt{C(t)L_2(t) \sum_{s \in U_{\Tilde{\boldsymbol{\pi}}, \hat{\boldsymbol{\pi}}}}  \frac{2|\Tilde{\pi}_s - \hat{\pi}_s|^2}{ \left| {\pi}^{*}_p - \Tilde{\pi}_p \right|^2 T_s(t)}} \nonumber \\
        & \leq & \frac{\Delta_{(p)}}{7} \label{third_inequality_1},
    \end{eqnarray}
    and similarly, 
    \begin{eqnarray}
        && \frac{1}{\left| {\pi}^{*}_p - \Tilde{\pi}_p \right|} \left| \sum_{s \in U_{\hat{\boldsymbol{\pi}}, \Tilde{\boldsymbol{\pi}}}} (\Tilde{\pi}_s - \hat{\pi}_s)(\theta^{k}_{s}(t) - \hat{\mu}_{s}(t)) \right| \nonumber \\
        &\leq& \sqrt{C(t)L_2(t) \sum_{s \in U_{\Tilde{\boldsymbol{\pi}}, \hat{\boldsymbol{\pi}}}}  \frac{2|\pi_s - \hat{\pi}_s|^2}{\left| {\pi}^{*}_p - \Tilde{\pi}_p \right|^2 T_s(t)}} \nonumber \\
        & \leq & \frac{\Delta_{(p)}}{7} \label{third_inequality_2}.
    \end{eqnarray}
    Therefore, for any $k$, we have that
    \begin{eqnarray}
        &&\frac{1}{\left| {\pi}^{*}_p - \Tilde{\pi}_p \right|} \sum_{s = 1}^{d} (\Tilde{\pi}_s - \hat{\pi}_s) \theta^{k}_s(t) \nonumber \\
        & = &
        \begin{split}
            \frac{1}{\left| {\pi}^{*}_p - \Tilde{\pi}_p \right|} \biggl( 
            &\sum_{s = 1}^{d} (\Tilde{\pi}_s - \hat{\pi}_s)
            \hat{\mu}_s(t) \\
            &+ \sum_{s \in U_{\Tilde{\boldsymbol{\pi}}, \hat{\boldsymbol{\pi}}}} (\Tilde{\pi}_s - \hat{\pi}_s) (\theta^{k}_s(t) - \hat{\mu}_s(t)) \\
            &+ \sum_{s \in U_{\hat{\boldsymbol{\pi}}, \Tilde{\boldsymbol{\pi}}}} (\hat{\pi}_s - \Tilde{\pi}_s) (\theta^{k}_s(t) - \hat{\mu}_s(t)) \biggr)
        \end{split}  \nonumber \\
        & \leq &
        \begin{split}
            \frac{1}{\left| {\pi}^{*}_p - \Tilde{\pi}_p \right|} \biggl( & \sum_{s = 1}^{d} (\Tilde{\pi}_s - \hat{\pi}_s)
            \hat{\mu}_s(t) \\
            &+ \biggl| \sum_{s \in U_{\Tilde{\boldsymbol{\pi}}, \hat{\boldsymbol{\pi}}}} (\Tilde{\pi}_s - \hat{\pi}_s) (\theta^{k}_s(t) - \hat{\mu}_s(t)) \biggr| \\
            &+ \biggl| \sum_{s \in U_{\hat{\boldsymbol{\pi}}, \Tilde{\boldsymbol{\pi}}}} (\hat{\pi}_s - \Tilde{\pi}_s) (\theta^{k}_s(t) - \hat{\mu}_s(t)) \biggr| \biggr)
        \end{split} \nonumber \\
        & \leq & 0 + \frac{2\Delta_{(p)}}{7} = \frac{2\Delta_{(p)}}{7}. \nonumber
     \end{eqnarray}
     This means that $\frac{\Tilde{\Delta}^{k_t^{*}}_t}{\left| \Tilde{\pi}_p - \hat{\pi}_p \right|} \leq \frac{2\Delta_{(p)}}{7}$. \par
     Moreover, since $\sum_{s = 1}^{d} \Tilde{\pi}_s \theta^{k^{*}_t}_s \geq \sum_{s = 1}^{d} \hat{\pi}_s \theta^{k^{*}_t}_s$, we have that
     \begin{eqnarray}
     && \frac{1}{\left| {\pi}^{*}_p - \Tilde{\pi}_p \right|} \sum_{s = 1}^{d} (\Tilde{\pi}_s - \hat{\pi}_s) \hat{\mu}_{s}(t) \nonumber \\
        & = &
        \begin{split}
            \frac{1}{\left| {\pi}^{*}_p - \Tilde{\pi}_p \right|} \biggl( & \sum_{s = 1}^{d} (\Tilde{\pi}_s - \hat{\pi}_s)
            \theta^{k^{*}_{t}}_s(t) \\ 
            &- \sum_{s \in U_{\Tilde{\boldsymbol{\pi}}, \hat{\boldsymbol{\pi}}}} (\Tilde{\pi}_s - \hat{\pi}_s) (\theta^{k^{*}_{t}}_s(t) - \hat{\mu}_s(t)) \\
            &+ \sum_{s \in U_{\hat{\boldsymbol{\pi}}, \Tilde{\boldsymbol{\pi}}}} (\hat{\pi}_s - \Tilde{\pi}_s) (\theta^{k^{*}_{t}}_s(t) - \hat{\mu}_s(t)) \biggl)
        \end{split} \nonumber \\
        & \geq &
        \begin{split}
        \frac{1}{\left| {\pi}^{*}_p - \Tilde{\pi}_p \right|} \biggl( & \sum_{s = 1}^{d} (\Tilde{\pi}_s - \hat{\pi}_s)
        \theta^{k^{*}_{t}}_s(t) \\
        &- \biggl| \sum_{s \in U_{\Tilde{\boldsymbol{\pi}}, \hat{\boldsymbol{\pi}}}} (\Tilde{\pi}_s - \hat{\pi}_s) (\theta^{k^{*}_{t}}_s(t)  - \hat{\mu}_s(t)) \biggr| \\
            &- \biggl| \sum_{s \in U_{\hat{\boldsymbol{\pi}}, \Tilde{\boldsymbol{\pi}}}} (\hat{\pi}_s - \Tilde{\pi}_s) (\theta^{k^{*}_{t}}_s(t) - \hat{\mu}_s(t)) \biggr| \biggr)
        \end{split} \nonumber \\
        & \geq & 0 - \frac{2\Delta_{(p)}}{7} =  -\frac{2\Delta_{(p)}}{7}. \nonumber
     \end{eqnarray}
     On the other hand, by event $\mathcal{E}_{2}$, we know that there exists a $k'$ such that $\sum_{s = 1}^{d} \pi^{*}_s \theta^{k'}_s(t) - \sum_{s = 1}^{d} \Tilde{\pi}_s \theta^{k'}_s(t) \geq \Delta_{\boldsymbol{\pi}^{*}, \Tilde{\boldsymbol{\pi}}}$. 
     Then,
     \begin{eqnarray}
        &&\frac{1}{\left| {\pi}^{*}_p - \Tilde{\pi}_p \right|} \left(\sum_{s = 1}^{d} \pi^{*}_s \theta^{k'}_s(t) - \sum_{s = 1}^{d} \hat{\pi}_s \theta^{k'}_s(t) \right) \nonumber \\
        & = &
        \begin{split}
           \frac{1}{\left| {\pi}^{*}_p - \Tilde{\pi}_p \right|} \Biggl( \biggl( & \sum_{s = 1}^{d} \pi^{*}_s \theta^{k'}_s(t) - \sum_{s = 1}^{d} \Tilde{\pi}_s \theta^{k'}_s(t) \biggr)  \\
            & + \biggl(\sum_{s = 1}^{d} \Tilde{\pi}_s \theta^{k'}_s(t) - \sum_{s = 1}^{d} \hat{\pi}_s \theta^{k'}_s(t) \biggr) \Biggr)
        \end{split} \nonumber \\
        & \geq &
            \frac{\Delta_{\boldsymbol{\pi}^{*}, \Tilde{\boldsymbol{\pi}}}}{\left| {\pi}^{*}_p - \Tilde{\pi}_p \right|}  + \frac{1}{\left| {\pi}^{*}_p - \Tilde{\pi}_p \right|} \biggl( \sum_{s = 1}^{d} \Tilde{\pi}_s \theta^{k'}_s(t) - \sum_{s = 1}^{d} \hat{\pi}_s \theta^{k'}_s(t) \biggr) \nonumber \\
        & \geq &
        \begin{split}
            \Delta_{(p)} + \frac{1}{\left| {\pi}^{*}_p - \Tilde{\pi}_p \right|} \biggl( & \sum_{s \in U_{\Tilde{\boldsymbol{\pi}}, \hat{\boldsymbol{\pi}}}} \left(\Tilde{\pi}_s(t) -  \hat{\pi}_s)(t) \right)  \theta^{k'}_s(t)  \\ 
            &- \sum_{s \in U_{\hat{\boldsymbol{\pi}}, \Tilde{\boldsymbol{\pi}}}} \left(\hat{\pi}_s (t) - \Tilde{\pi}_s(t)\right) \theta^{k'}_s(t) \biggr)
        \end{split} \nonumber \\
        & = &
         \begin{split} 
               & \Delta_{(p)}  + \frac{1}{\left| {\pi}^{*}_p - \Tilde{\pi}_p \right|} \biggl( \sum_{s \in U_{\Tilde{\boldsymbol{\pi}}, \hat{\boldsymbol{\pi}}}} \left(\Tilde{\pi}_s(t) -  \hat{\pi}_s)(t) \right)  \hat{\mu}_s(t) \\
               & -  \sum_{s \in U_{\hat{\boldsymbol{\pi}}, \Tilde{\boldsymbol{\pi}}}} \left(\hat{\pi}_s (t) - \Tilde{\pi}_s(t)\right) \hat{\mu}_s(t) \biggr)  \\
               & + \frac{1}{\left| {\pi}^{*}_p - \Tilde{\pi}_p \right|} \left(\sum_{s \in U_{\Tilde{\boldsymbol{\pi}}, \hat{\boldsymbol{\pi}}}} \left(\Tilde{\pi}_s(t) -  \hat{\pi}_s)(t) \right)  (\theta^{k'}_{s}(t) - \hat{\mu}_s(t)) \right) \\
               & - \frac{1}{\left| {\pi}^{*}_p - \Tilde{\pi}_p \right|} \biggl( \sum_{s \in U_{\hat{\boldsymbol{\pi}}, \Tilde{\boldsymbol{\pi}}}} \left(\hat{\pi}_s (t) - \Tilde{\pi}_s(t)\right) (\theta^{k'}_{s}(t) - \hat{\mu}_s(t)) \biggr)
          \end{split} \nonumber \\
        & \geq  &
        \begin{split} 
               &\Delta_{(p)} - \frac{2\Delta_{(p)}}{7}  \\
               &- \frac{1}{\left| {\pi}^{*}_p - \Tilde{\pi}_p \right|} \left|\sum_{s \in U_{\Tilde{\boldsymbol{\pi}}, \hat{\boldsymbol{\pi}}}} \left(\Tilde{\pi}_s(t) -  \hat{\pi}_s(t) \right)  (\theta^{k'}_{s}(t) - \hat{\mu}_s(t)) \right| \\
               &- \frac{1}{\left| {\pi}^{*}_p - \Tilde{\pi}_p \right|} \left| \sum_{s \in U_{\hat{\boldsymbol{\pi}}, \Tilde{\boldsymbol{\pi}}}} \left(\hat{\pi}_s (t) - \Tilde{\pi}_s(t) \right) (\theta^{k'}_{s}(t) - \hat{\mu}_s(t)) \right|
        \end{split} \nonumber \\
        & \geq & \Delta_{(p)} - \frac{4\Delta_{(p)}}{7} \nonumber \\
        & > & \frac{2\Delta_{(p)}}{7}. \nonumber
     \end{eqnarray}
     This means that $\frac{\Tilde{\Delta}^{k_t^{*}}_t}{\left| \Tilde{\pi}_p - \hat{\pi}_p \right|} > \frac{2\Delta_{(p)}}{7}$, which contradicts with $\frac{\Tilde{\Delta}^{k_t^{*}}_t}{\left| \Tilde{\pi}_p - \hat{\pi}_p \right|} \leq \frac{2\Delta_{(p)}}{7}$. \par
     
     Case 2: $\hat{{\pi}}_p \neq {\pi}^{*}_p$. 
     In this case, $\Delta_{(p)} \leq \frac{\boldsymbol{\mu}^{\top} \left( \boldsymbol{\pi}^{*} - \hat{\boldsymbol{\pi}} \right)}{\left|{\pi}_p^{*} - \hat{\pi}_p(t) \right| }$. 
     By event $\mathcal{E}_2$, we know that there exists a $k$ such that $\sum_{s = 1}^{d} \pi^{*}_s \theta_s^{k}(t) - \sum_{s = 1}^{d} \hat{\pi}_s(t) \theta^{k}_s(t) \geq \Delta_{\boldsymbol{\pi}^{*}, \hat{\boldsymbol{\pi}}}$. 
     Hence, $\Tilde{\Delta}_t^k \geq \Delta_{\boldsymbol{\pi}^{*}, \hat{\boldsymbol{\pi}}}$.  
     Moreover, since $k^{*}_t = \argmax_k \Tilde{\Delta}_t^k$, we have that $ \Tilde{\Delta}^{k^{*}_t}_t = \sum_{s = 1}^{d} \pi^{*}_s \theta_s^{k^{*}_t}(t) - \sum_{s = 1}^{d} \hat{\pi}_s(t) \theta^{k^{*}_t}_s(t) \geq \Delta_{\boldsymbol{\pi}^{*}, \hat{\boldsymbol{\pi}}}$, 
     which is the same as $\sum_{s \in U(\boldsymbol{\pi}^{*}, \hat{\boldsymbol{\pi}})} (\pi^{*}_s - \hat{\pi}_s) \theta_s^{k^{*}_t}(t) - \sum_{s = 1}^{d} (\hat{\pi}_s(t) - \pi^{*}) \theta^{k^{*}_t}_s(t) \geq \Delta_{\boldsymbol{\pi}^{*}, \hat{\boldsymbol{\pi}}}$. 
     On the other hand, by event $\mathcal{E}_1$, we also have that, for any $k$, 
    \begin{eqnarray}
        && \frac{1}{\left| {\pi}^{*}_p - \hat{\pi}_p \right|} \left| \sum_{s \in U_{\Tilde{\boldsymbol{\pi}}, \hat{\boldsymbol{\pi}}}} (\Tilde{\pi}_s - \hat{\pi}_s)(\theta^{k^{*}_t}_{s}(t) - \hat{\mu}_{s}(t)) \right| \nonumber \\
        &\leq& \frac{1}{\left| {\pi}^{*}_p - \hat{\pi}_p \right|} \sqrt{C(t)L_2(t) \sum_{s \in U_{\Tilde{\boldsymbol{\pi}}, \hat{\boldsymbol{\pi}}}}  \frac{2|\Tilde{\pi}_s - \hat{\pi}_s|^2}{T_s(t)}} \nonumber \\
        & \leq & \sqrt{C(t)L_2(t) \sum_{s \in U_{\Tilde{\boldsymbol{\pi}}, \hat{\boldsymbol{\pi}}}}  \frac{2|\Tilde{\pi}_s - \hat{\pi}_s|^2}{\left| {\pi}^{*}_p - \hat{\pi}_p \right|^2 T_s(t)}} \nonumber \\
        & \leq & \frac{\Delta_{(p)}}{7}, \nonumber 
    \end{eqnarray}
    and similarly, 
    \begin{eqnarray}
        && \frac{1}{\left| {\pi}^{*}_p - \hat{\pi}_p \right|} \left| \sum_{s \in U_{\hat{\boldsymbol{\pi}}, \Tilde{\boldsymbol{\pi}}}} (\Tilde{\pi}_s - \hat{\pi}_s)(\theta^{k^{*}}_{s}(t) - \hat{\mu}_{s}(t)) \right| \nonumber \\
        &\leq& \frac{1}{\left| {\pi}^{*}_p - \hat{\pi}_p \right|} \sqrt{C(t)L_2(t) \sum_{s \in U_{\Tilde{\boldsymbol{\pi}}, \hat{\boldsymbol{\pi}}}}  \frac{2|\Tilde{\pi}_s - \hat{\pi}_s|^2}{T_s(t)}} \nonumber \\
        & \leq & \sqrt{C(t)L_2(t) \sum_{s \in U_{\hat{\boldsymbol{\pi}}, \Tilde{\boldsymbol{\pi}}}}  \frac{2|\Tilde{\pi}_s - \hat{\pi}_s|^2}{\left| {\pi}^{*}_p - \hat{\pi}_p \right|^2 T_{p}}} \nonumber \\
        & \leq & \frac{\Delta_{{(p)}}}{7}. \nonumber 
    \end{eqnarray}
    Therefore, 
    \begin{eqnarray}
       && \frac{1}{\left| {\pi}^{*}_p - \hat{\pi}_p \right|} \left( \sum_{s \in U_{\Tilde{\pi}, \hat{\boldsymbol{\pi}}}} (\Tilde{\pi}_s - \hat{\pi}_s) \hat{\mu}_s - \sum_{s \in U_{\hat{\boldsymbol{\pi}}, \Tilde{\boldsymbol{\pi}}}} \left(\hat{\pi}_s - \Tilde{\pi}_s \right) \hat{\mu}_s \right) \nonumber \\
       & = &
         \begin{split} 
                \frac{1}{\left| {\pi}^{*}_p - \hat{\pi}_p \right|} \Biggl( \biggl(&\sum_{s \in U_{\Tilde{\boldsymbol{\pi}}, \hat{\boldsymbol{\pi}}}} \left(\Tilde{\pi}_s(t) -  \hat{\pi}_s(t) \right)  \theta^{k^{*}_t}_s(t) \\ 
               &- \sum_{s \in U_{\hat{\boldsymbol{\pi}}, \Tilde{\boldsymbol{\pi}}}} \left(\hat{\pi}_s (t) - \Tilde{\pi}_s(t)\right) \theta^{k^{*}_t}_s(t) \biggr)  \\
               &- \biggl(\sum_{s \in U_{\Tilde{\boldsymbol{\pi}}, \hat{\boldsymbol{\pi}}}} \left(\Tilde{\pi}_s(t) -  \hat{\pi}_s)(t) \right)  (\theta^{k^{*}_t}_{s}(t) - \hat{\mu}_s(t)) \biggr) \\
               &+ \biggl( \sum_{s \in U_{\hat{\boldsymbol{\pi}}, \Tilde{\boldsymbol{\pi}}}} \left(\hat{\pi}_s (t) - \Tilde{\pi}_s(t)\right) (\theta^{k^{*}_t}_{s}(t) - \hat{\mu}_s(t)) \biggr) \Biggl)
          \end{split} \nonumber \\
        & = &
         \begin{split} 
               & \frac{1}{\left| {\pi}^{*}_p - \hat{\pi}_p \right|} \biggl(\sum_{s \in U_{\Tilde{\boldsymbol{\pi}}, \hat{\boldsymbol{\pi}}}} \left(\Tilde{\pi}_s(t) -  \hat{\pi}_s(t) \right)  \theta^{k^{*}_t}_s(t) \\
               & -  \sum_{s \in U_{\hat{\boldsymbol{\pi}}, \Tilde{\boldsymbol{\pi}}}} \left(\hat{\pi}_s (t) - \Tilde{\pi}_s(t)\right) \theta^{k^{*}_t}_s(t) \biggr)  \\
               & - \frac{1}{\left| {\pi}^{*}_p - \hat{\pi}_p \right|} \left|\sum_{s \in U_{\Tilde{\boldsymbol{\pi}}, \hat{\boldsymbol{\pi}}}} \left(\Tilde{\pi}_s(t) -  \hat{\pi}_s)(t) \right)  (\theta^{k^{*}_t}_{s}(t) - \hat{\mu}_s(t)) \right| \\
               & - \frac{1}{\left| {\pi}^{*}_p - \hat{\pi}_p \right|} \left| \sum_{s \in U_{\hat{\boldsymbol{\pi}}, \Tilde{\boldsymbol{\pi}}}} \left(\hat{\pi}_s (t) - \Tilde{\pi}_s(t)\right) (\theta^{k^{*}_t}_{s}(t) - \hat{\mu}_s(t)) \right|
          \end{split} \nonumber \\
        & \geq & \frac{\Delta_{\boldsymbol{\pi}^{*}, \hat{\boldsymbol{\pi}}}}{\left| {\pi}^{*}_p - \hat{\pi}_p \right|} - \frac{2\Delta_{(p)}}{7} \nonumber \\
        & \geq & \Delta_{(p)} - \frac{2\Delta_{(p)}}{7} \nonumber \\
        & > & 0 \nonumber 
    \end{eqnarray} 
    This means that $\sum_{s 1}^{d} \Tilde{\pi}_s \hat{\mu}_s - \sum_{s = 1}^{d} \hat{\pi}_s \hat{\mu}_s  > 0$, which contradicts with the fact that $\hat{\boldsymbol{\pi}} = \argmax_{\boldsymbol{\pi} \in \mathcal{A}} \boldsymbol{\mu}^{\top}\boldsymbol{\pi}$.
\end{proof}
We can obtain Theorem \ref{TSUpperBoundTheorem_Naive} following the same discussion as the proof of Theorem \ref{TSUpperBoundTheorem_R-CPE-MAB}.
\section{Proof of Proposition \ref{TighterResultThanWangAndZhu}}
We show only for (\ref{H^R_result}) since we can follow the same discussion for (\ref{H^N_result}).
In the ordinary CPE-MAB, since any element in an action is either 0 or 1, we have
\begin{eqnarray}
    V &=& \max_{s \in [d]} V_s \nonumber \\ 
      &=& \max_{ s\in[d], \boldsymbol{\pi}' \in \mathcal{A}, \boldsymbol{\pi} \in \{ \boldsymbol{\pi} \in \mathcal{A} \ | \ \pi^{*}_s \neq \pi_s \}}  \frac{\left| \pi_s - \pi'_s \right|}{\left| \pi^{*}_s - \pi_s \right|^2} \sum_{e = 1}^{d} \left| \pi_e - \pi'_e \right|  \nonumber \\
      & = & \max_{\boldsymbol{\pi}, \boldsymbol{\pi}' \in \mathcal{A}} \sum_{s = 1}^{d} \left| \pi_s - \pi'_s \right|  \nonumber \\
      & = & \mathrm{width}, 
\end{eqnarray}
and
\begin{eqnarray}
    \mathbf{H}^{\mathrm{R}} = \sum_{s = 1}^d \frac{V_s}{\Delta^2_{(s)}} = \sum_{s = 1}^{d} \frac{V_s}{\Delta^2_s}.
\end{eqnarray}
Therefore, the upper bound shown in (\ref{H^R_result}) is tighter than that of \citet{WangAndZhu2022} since
\begin{eqnarray}
    \mathbf{H}^{\mathrm{R}} 
    &=& \sum_{s = 1}^{d} \frac{V_s}{ \Delta^2_s} \nonumber \\
    &\leq& V \sum_{s = 1}^{d} \frac{1}{\Delta^2_s}  \nonumber \\
    & = & \mathrm{width} \sum_{s = 1}^{d} \frac{1}{\Delta^2_s} .
\end{eqnarray}
\end{document}